%% file: TVT-FedSL.tex
\documentclass[journal]{IEEEtran}

\usepackage[ruled,vlined,linesnumbered]{algorithm2e}
\usepackage{verbatim}
\usepackage{amsfonts}
\usepackage{amssymb}
\usepackage{stfloats}
\usepackage{cite}
\usepackage{psfrag}
\usepackage{amsmath}
\usepackage{array}
\usepackage{epstopdf}
\usepackage{authblk}
\usepackage{amsthm} 
\usepackage{lipsum}
\usepackage{authblk}
\usepackage{mathtools}
\usepackage{cuted}
\usepackage{booktabs}
\usepackage{bm}

\usepackage{graphicx}  
\usepackage{float}  
\usepackage{subfigure}  

\allowdisplaybreaks[2]

\newcommand{\xqedhere}[2]{%
\rlap{\hbox to#1{\hfil\llap{\ensuremath{#2}}}}}

\include{header}

\usepackage{amsmath}

\usepackage{array}

\usepackage{url}
\usepackage{graphicx,amsmath,amssymb,amsfonts}

\usepackage{hyperref}
\hypersetup{bookmarks=true,
					 bookmarksopen=true,
					 bookmarksnumbered=true,
					 colorlinks  =true,
					 breaklinks =true,
					 linkcolor=black,
					 citecolor=black,
					 urlcolor =black,
				}

\usepackage{xcolor}

\makeatletter

\newcommand{\Rmnum}[1]{\expandafter\@slowromancap\romannumeral #1@}
\makeatother

\usepackage{cancel}  

\usepackage{lipsum}

\begin{document}
	\bstctlcite{ref:BSTcontrol}
	
	\title{Federated Split Learning with Model Pruning and Gradient Quantization in Wireless Networks}

\author{~Junhe~Zhang,~Wanli~Ni,~Dongyu~Wang}
	
\maketitle

\newcommand\blfootnote[1]{%
	\begingroup 
	\renewcommand\thefootnote{}\footnote{#1}%
	\addtocounter{footnote}{-1}%
	\endgroup 
}

\blfootnote{	
\textit{(Corresponding author: Dongyu Wang.)}

Junhe Zhang and Dongyu Wang are with the Key Laboratory
of Universal Wireless Communication, Ministry of Education, Beijing University of Posts and Telecommunications, Beijing 100876, China (e-mail: jhzhangbupt@163.com; dy\_wang@bupt.edu.cn).

Wanli Ni is with the Department of Electronic Engineering, Tsinghua University, China (e-mail: niwanli@tsinghua.edu.cn).
}
\vspace{-4 mm}

\begin{abstract}
As a paradigm of distributed machine learning, federated learning typically requires all edge devices to train a complete model locally. However, with the increasing scale of artificial intelligence models, the limited resources on edge devices often become a bottleneck for efficient fine-tuning. 
To address this challenge, federated split learning (FedSL) implements collaborative training across the edge devices and the server through model splitting.
In this paper, we propose a lightweight FedSL scheme, that further alleviates the training burden on resource-constrained edge devices by pruning the client-side model dynamicly and using quantized gradient updates to reduce computation overhead. Additionally, we apply random dropout to the activation values at the split layer to reduce communication overhead. We conduct theoretical analysis to quantify the convergence performance of the proposed scheme. Finally, simulation results verify the effectiveness and advantages of the proposed lightweight FedSL in wireless network environments.
\end{abstract}


\section{Introduction}
As the Internet of Things (IoT) landscape rapidly expands, wireless edge devices are now generating colossal volumes of data, fostering the integration of artificial intelligence (AI) within wireless networks. Federated learning (FL), a groundbreaking distributed machine learning approach \cite{bg1}, has emerged as a solution that challenges traditional centralized learning paradigms. Distinctively, FL enables each edge device to retain its dataset locally, facilitating distributed training through the exchange of solely local model parameters, thereby safeguarding privacy. Nevertheless, the escalating complexity and size of AI models pose significant challenges for edge devices, which often operate with constrained computational and communication resources. Consequently, these devices struggle to locally train full models, rendering on-device training for AI-powered, computationally intensive tasks impractical.

Split learning (SL) tackles FL's challenges by partitioning deep neural networks (DNNs) between edge devices and an edge server. This setup fosters collaborative training, with intermediate data transmitted wirelessly \cite{sl_bg1}. By shifting computational load to the server, SL effectively alleviates the computation burden of edge devices \cite{wanli}. However, early SL implementations followed a sequential training paradigm, causing delays and inefficiencies due to the absence of parallel computing.

Recently, federated split learning (FedSL) framework has been proposed \cite{splitfed}, which harnesses model splitting and parallel computing to enhance training efficiency across edge devices. To further expedite model training, the authors in \cite{epsl} proposed a parallel SL framework that streamlined backpropagation by diminishing activation gradient dimensions through last-layer gradient aggregation. In addition, the authors in \cite{sgsl} introduced a methodology without additional client-to-client communication by broadcasting a universally averaged gradient at the split layer.
The authors in \cite{accsfl} considered devices with individual split points, and analyzed how split point selection influences convergence and latency. Similarly, the authors in \cite{adaptsfl} analyzed the impact of client-side model splitting and aggregation, aiming to minimize convergence time over edge networks.

However, most of existing FedSL schemes have overlooked the storage and computational burdens imposed by deploying and fine-tuning large AI models on resource-constrained devices. Additionally, the transmission of the smashed data between devices and the base station (BS) introduces substantial communication overhead, especially when the batch size and the dimensionality of the intermediate features are large. To reduce the computational and communication overhead of existing FedSL on resource-constrained devices, we propose a lightweight FedSL scheme, which allivates training burden by pruning model parameters, quantizing gradients, and randomly dropping intermediate activation values during wireless transmission.
The main contributions of this work are as follows:
\begin{itemize}
	\item
	We propose a lightweight FedSL scheme that alleviates the storage and computational load on resource-constrained devices when fine-tuning large models. Meanwhile, a random dropout mechanism is applied on the intermediate activation values to reduce the communication overhead.
	\item 
	We undertake a theoretical analysis by deriving the convergence upper bound of the lightweight FedSL scheme. Our findings reveal that while a higher pruning rate can potentially hinder convergence, strategies such as more frequent aggregation and employing a shallower split layer can counteract this effect, facilitating faster model convergence.
	\item
	We substantiate our claims through numerical simulations, demonstrating the superiority of our proposed scheme. Notably, the integration of pruning and quantization in lightweight FedSL not only accelerates training but also serves as an effective regularization technique, enhancing model performance, particularly in scenarios where neural networks exhibit substantial redundancy. Moreover, moderate dropout rate can reduce communication latency without sacrificing too much performance.
\end{itemize}

\section{System Model}

\begin{figure}
	\centering
	\includegraphics[scale=0.45]{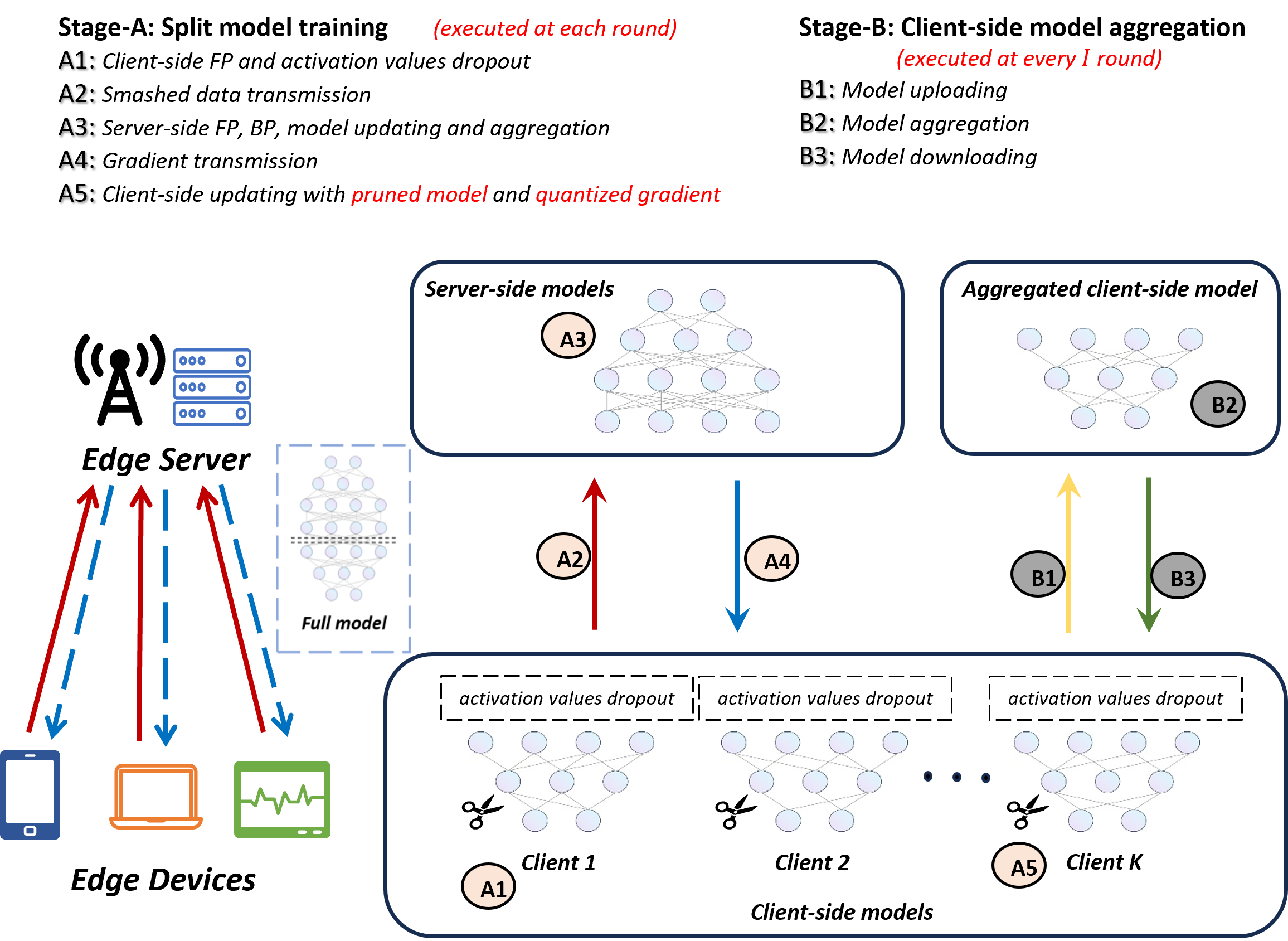}
	\caption{An illustration of the proposed FedSL with client-side model pruning and gradient quantization.}
	\vspace{-3mm}
	\label{fig:sysmodel}
\end{figure}

We consider a wireless network consisting of a set of $\mathcal{K}=\{1, 2,\ldots,K\}$ devices also known as clients and a BS connected with an edge server. The full model $\mathbf{w}$ with $L$ layers is divided into the client-side model $\mathbf{w}_{\mathbf{c},k}$ and the server-side model $\mathbf{w}_{\mathbf{s},k}$.
The client-side model $\mathbf{w}_{\mathbf{c},k}$ retains neurons of layers $\{1, 2,\ldots,L_c\}$ while the server-side model $\mathbf{w}_{\mathbf{s},k}$ retains neurons of layers $\{L_{c}+1, L_{c}+2,\ldots,L\}$, where $L_c$ is the split layer. 

The training objective is to minimize the global loss function, which is given by
\begin{equation}\label{problem1}
\min_{\mathbf{w}}F(\mathbf{w}):=\min_{\mathbf{w}_{c,k};\mathbf{w}_{s,k}}\frac{1}{K}\sum_{k=1}^{K}F(\mathbf{w}_{\mathbf{c},k};\mathbf{w}_{\mathbf{s},k};\mathcal{D}_k),
\end{equation}
where $F(\mathbf{w}_{\mathbf{c},k};\mathbf{w}_{\mathbf{s},k};\mathcal{D}_k)$ is the loss function of client $\mathit{k}$ over local datasets $\mathcal{D}_k$. Due to the identical gradient updates applied in each round for the server-side model, i.e., for every client, their corresponding server-side model remains the same, we use $\mathbf{w}_{\mathbf{s}}$ to replace $\mathbf{w}_{\mathbf{s},k}$ to simplify the expression in most cases.

\subsection{Lightweight FedSL}
As is shown in Fig.~\ref{fig:sysmodel}, we consider a lightweight FedSL system where the server connects with clients via wireless channels. Due to the constrained communication and computing resources, the full model is divided into two parts. Meanwhile, in order to further reduce the computation and communication burden of the clients, each client performs model pruning and gradient quantization, applying dropout to activation values before sending them to the server. Dropout randomly sets a fraction of activation values to zero in each training round, effectively reducing transmitted data and mitigating overfitting. The client-side models are aggregated every $I$ rounds to attain a global model. Each client initializes a binary mask $\mathbf{m}_k$ for pruning before the iterative training.

For the iterative split training stage, each client performs forward propagation and sends the activation values along with corresponding labels to the server after dropout. The server then continues model training, updates the server-side model, and performs server-side model aggregation. Then the clients perform pruning and update their models with quantized gradients. For the client-side model aggregation stage, which is performed every $I$ rounds, all clients upload their models for aggregation to obtain a global model.
Specifically, for each training round $t\in\mathcal{T}=\{1, 2,...,T\}$, five main steps are considered as follows.

\subsubsection{Client-Side Forward Propagation and Activation Values Dropout}
Each client $k \in \mathcal{K}$ performs forward propagation (FP) by computing the activations $\bm{A}_{\mathbf{c},k,t}^{(l)}$ of each layer $l$:
\begin{align}
	\bm{A}_{\mathbf{c},k,t}^{(l)}\!=\!\psi\left(\mathbf{w}_{\mathbf{c},k,t}^{(l)}\bm{A}_{\mathbf{c},k,t}^{(l-1)}+b_{\mathbf{c},k,t}^{(l)}\right)\!,\! \quad \! \forall l \in \{1,\ldots,L_c\},
\end{align}
where $\psi$ is the activation function, $\mathbf{w}_{\mathbf{c},k,t}^{(l)}$ and $b_{\mathbf{c},k,t}^{(l)}$ are the weights and biases of the $l$-th layer, respectively.
To reduce the communication overhead of transmitting intermediate activation values, we apply dropout to the activation values. The dropout process for the $i$-th feature vector ${\bm{a}}_{\mathbf{c},k,t,i}^{(L_c)}$ of the activation value ${\bm{A}}_{\mathbf{c},k,t}^{(L_c)}$ from the split layer $L_c$ can be expressed as
\begin{align}
	\tilde{\bm{a}}_{\mathbf{c},k,t,i}^{(L_c)} = \frac{\delta_i}{1-p_i} {\bm{a}}_{\mathbf{c},k,t,i}^{(L_c)},
\end{align}
where the dropout probability $p_i$ is a is a predefined hyperparameter, $\delta_i \!\! \in \!\! \{0,1\}$ is a Bernoulli random variable with $\mathbb{P}(\delta_i \!\! = \!\! 1) \! = \! (1 \! - \! p_i)$, indicating whether ${\bm{a}}_{\mathbf{c},k,t,i}^{(L_c)}$ is dropped or not.
To preserve the expected output of the whole model, the retained features vectors are multiplied by $\frac{1}{1-p_i}$.
\begin{corollary}\label{dropout}
	The expected value of the feature vector remains invariant after the dropout operation, i.e., $\mathbb{E}[\tilde{\bm{a}}_{\mathbf{c},k,t,i}^{(L_c)}] = \mathbb{E}[{\bm{a}}_{\mathbf{c},k,t,i}^{(L_c)}]$.
\end{corollary}

\begin{proof}
	The expectation of the feature vector $\tilde{\bm{a}}_{\mathbf{c},k,t,i}^{(L_c)}$ after dropout can be expressed as
	\begin{align*}
		&\mathbb{E}[\tilde{\bm{a}}_{\mathbf{c},k,t,i}^{(L_c)}] = \mathbb{E}[\delta_i \cdot \frac{1}{1-p_i} {\bm{a}}_{\mathbf{c},k,t,i}^{(L_c)}] \\
		&=\mathbb{P}(\delta_i=1)\cdot\frac1{1-p_i}{\bm{a}}_{\mathbf{c},k,t,i}^{(L_c)}+\mathbb{P}(\delta_i=0)\cdot0 \\
		&= (1-p_i) \cdot \frac{1}{1-p_i} {\bm{a}}_{\mathbf{c},k,t,i}^{(L_c)} = {\bm{a}}_{\mathbf{c},k,t,i}^{(L_c)}.
		\tag*{\qed}
	\end{align*}
\end{proof}

The smashed data $\bm{S}_{\mathbf{c},k,t}=\big(\tilde{\bm{A}}_{\mathbf{c},k,t}^{(L_c)};\mathcal{H}_k\big)$  including the activation value $\tilde{\bm{A}}_{\mathbf{c},k,t}^{(L_c)}$ and their corresponding label $\mathcal{H}_k$, are transmitted to the server for subsequent training.

\subsubsection{Server-Side Forward Propagation and Backward Propagation}
Upon receiving the smashed data $S_{\mathbf{c},k,t}$ from all clients, the server continues the forward propagation through the remaining layers of the model:
\begin{align}
	\bm{A}_{\mathbf{s},k,t}^{(l)}\!=\!\psi\left(\mathbf{w}_{\mathbf{s},k,t}^{(l)}\bm{A}_{\mathbf{s},k,t}^{(l-1)}+b_{\mathbf{s},k,t}^{(l)}\right)\!, \! \quad \!
	\forall l \in \{L_c\!+\!1,\ldots,L\},
\end{align}
where $\bm{A}_{\mathbf{s},k,t}^{(L_c)} = \tilde{\bm{A}}_{\mathbf{c},k,t}^{(L_c)}$. 
Then, the server calculates the loss based on the difference between final response and the labels uploaded by the client and performs backward propagation to compute the gradients with respect to the server-side weights:
\begin{align}
	\mathbf{g}_{\mathbf{s},k,t}=\nabla F\left(\mathbf{w}_{\mathbf{s},k,t}\right) = \frac{\partial F(\mathbf{w})}{\partial \mathbf{w}_{\mathbf{s},k,t}},
\end{align}
The gradient for activation values $\frac{\partial F(\mathbf{w})}{\partial \tilde{\bm{A}}_{\mathbf{c},k,t}^{(L_c)}}$ are sent to the $k$-th client.
Then, the server aggregates the server-side models at every round after updating each of them. Specifically, the server-side model is updated as 
\begin{equation}\label{ws}
	\mathbf{w}_{\mathbf{s},t+1}= \frac{1}{K} \sum_{k=1}^{K}(\mathbf{w}_{\mathbf{s},k,t}-\eta{\mathbf{g}_{\mathbf{s},k,t}}),
\end{equation}
where $\eta$ is the learning rate.
Note that the server-side model for each client is identical in every round, the expression of server-side model updating can be simplified as
\begin{align}
	\mathbf{w}_{\mathbf{s},t+1} = \mathbf{w}_{\mathbf{s},t} - \frac{\eta}{K} \sum_{k=1}^{K} \mathbf{g}_{\mathbf{s},k,t}.
\end{align}

\subsubsection{Client-Side Model Pruning}
Upon receiving gradients $\mathbf{g}_{\mathbf{s},k,t}$ from the server, each client performs backpropagation to obtain gradients for cliet-side model:
\begin{align}
	\mathbf{g}_{\mathbf{c},k,t}=\nabla F\left(\mathbf{w}_{\mathbf{c},k,t}\right)
	=\frac{\partial F(\mathbf{w})}{\partial \tilde{\bm{A}}_{\mathbf{c},k,t}^{(L_c)}} \cdot \frac{\partial \tilde{\bm{A}}_{\mathbf{c},k,t}^{(L_c)}}{\partial \mathbf{w}_{\mathbf{c},k,t}}.
\end{align}

An importance estimation-based pruning strategy is employed for client-side model pruning \cite{Importance_Estimation}. The importance matrix of weights is defined by estimating the change in loss before and after pruning. The smaller change in loss function when a weight is set to zero, the lower its importance. By removing weights with lower importance, the neural network can be effectively lightened without significant loss in accuracy. Specifically, the importance matrix of the $k$-th client-side
model is calculated as
\begin{align}\label{importance}
	\mathcal{I}_{\mathbf{c},k,t} = |\mathbf{w}^\mathsf{T}_{\mathbf{c},k,t} \mathbf{g}_{\mathbf{c},k,t}|,
\end{align}
which approximate the absolute change of the loss given the removal of $\mathbf{w}_{\mathbf{c},k,t}$. The definition of the importance matrix $\mathcal{I}_{\mathbf{c},k,t}$ is derived from the first-order Taylor expansion of the loss function $F(\cdot)$ with respect to $\mathbf{w}_{\mathbf{c},k,t}$. Unlike pruning schemes based on weight magnitude, pruning based on importance estimation quantifies the impact of weights on the output, which avoids pruning weights with small magnitudes but significantly influence the results \cite{platon}.

Furthermore, we adopt a dynamic pruning scheme based on progressive sparsity during training, rather than a fixed threshold \cite{zhu2017prune}. The binary mask $\mathbf{m}_{k}$ depends on the target sparsity of the current iteration, named pruning rate $\rho_t$. Specifically, setting the initial sparsity $\rho_{0}=0$, the target sparsity in the $t$-th iteration can be denoted as
\begin{equation}
    \rho_t=\rho_f + (\frac{t}{T}-1)^3 \rho_f, \forall t\in \mathcal{T},
\end{equation}
 where $\rho_f$ is the preset final sparsity. This dynamic pruning scheme can rapidly remove redundant weights in the early stages of training and gradually reduce the pruning rate as training stabilizes, thereby enabling the model to converge more effectively.

After clients preform backward propagation, each of them calculates the sparsity of the current weight matrix $\mathbf{w}_{\mathbf{c},k,t}$. If the weight matrix does not reach the target sparsity $\rho_t$, the calculation of the binary mask $\mathbf{m}_{k,t}$ is triggered. The importance matrix of all client-side weights $\mathcal{I}_{\mathbf{c},k,t}$ is computed according to \eqref{importance} and then sorted. Setting the weights with low importance scores to zero to let the pruned weight matrix $\Tilde{\mathbf{w}}_{\mathbf{c},k,t}$ achieve the target sparsity. The pruning is performed as $\Tilde{\mathbf{w}}_{\mathbf{c},k,t}=\mathbf{m}_{k,t} \odot \mathbf{w}_{\mathbf{c},k,t}$ and $\Tilde{\mathbf{g}}_{\mathbf{c},k,t} = \mathbf{m}_{k,t} \odot \mathbf{g}_{\mathbf{c},k,t}$.

\begin{algorithm}[t]
	\LinesNumbered
	\SetAlgoLined
	\caption{Lightweight FedSL Scheme}
	\For{each round $t \in \{1, 2, \dots, T\}$}{
		\For{each client $k \in \{1, 2, \dots, K\}$}{
			Forward propagation to obtain $\bm{A}_{\mathbf{c},k,t}^{(l)}$;\\
			Activation values dropout $\bm{A}_{\mathbf{c},k,t}^{(l)} \rightarrow \tilde{\bm{A}}_{\mathbf{c},k,t}^{(l)}$; \\
			Transmitting smashed data $\bm{S}_{\mathbf{c},k,t}$ to server;
		}
		
		Server performs forward propagation to obtain $\bm{A}_{\mathbf{s},k,t}^{(L)}$; \\
		Loss calculation with $\bm{A}_{\mathbf{s},k,t}^{(L)}$ and $\mathcal{H}_k$; \\
		Calculating gradient $\mathbf{g}_{\mathbf{s},k,t}$, sending gradients of the activation values to client; \\
		Server-side model updating: 
		$\mathbf{w}_{\mathbf{s},t+1} = \mathbf{w}_{\mathbf{s},t} - \frac{\eta}{K} \sum_{k=1}^{K} \mathbf{g}_{\mathbf{s},k,t}$;
		
		
		\For{each client $k \in \{1, 2, \dots, K\}$}{
			Calculating gradient $\mathbf{g}_{\mathbf{c},k,t}$;\\
			\If{$\text{Sparsity}(\mathbf{w}_{\mathbf{c},k,t}) < \rho_t$}{
				Calculating importance matrix: $\mathcal{I}_{\mathbf{c},k,t}=|\mathbf{w}^\mathsf{T}_{\mathbf{c},k,t} \mathbf{g}_{\mathbf{c},k,t}|$; \\
				Calculating pruning mask $\mathbf{m}_{k,t}$; \\
				Model pruning: 
				$\tilde{\mathbf{w}}_{\mathbf{c},k,t}=\mathbf{m}_{k,t} \odot \mathbf{w}_{\mathbf{c},k,t}$ and $\tilde{\mathbf{g}}_{\mathbf{c},k,t} = \mathbf{m}_{k,t} \odot \mathbf{g}_{\mathbf{c},k,t}$;
			}
			Gradient quantization: $\tilde{\mathbf{g}}'_{\mathbf{c},k,t} = Q(\tilde{\mathbf{g}}_{\mathbf{c},k,t})$; \\
			Model updating: $\mathbf{w}_{\mathbf{c},k,t+1}=\tilde{\mathbf{w}}_{\mathbf{c},k,t}-\eta\tilde{\mathbf{g}}'_{\mathbf{c},k,t}$;
		}
		
		
		\If{$T \mid I$}{
			Each client uploads $\mathbf{w}_{\mathbf{c},k,t+1}$ to the server; \\
			Server aggregates received model: $\mathbf{w}_{\mathbf{c},t+1}=\frac1K\sum_{k=1}^K\mathbf{w}_{\mathbf{c},k,t+1}$; \\
			Server transmits $\mathbf{w}_{\mathbf{c},t+1}$ to all clients;
		}
	}
\end{algorithm}

\subsubsection{Client-Side Model Updating}
After obtaining the pruned weight matrix $\Tilde{\mathbf{w}}_{\mathbf{c},k,t}$ and the gradient $\Tilde{\mathbf{g}}_{\mathbf{c},k,t}$, we utilize the gradient quantization method proposed in \cite{amiri2020federated, reisizadeh2020fedpaq} to further reduce the computational complexity of client-side model updating.

For the gradient in $t$-th iteration  $\mathbf{g}_t\in\mathbb{R}^M$. We quantize the gradients $g_{t,m}$ to $g'_{t,m}$ with $q$ bits, where $m\in\{1, 2,..., M\}$. The lower and upper bounds of absolute values of $\mathbf{g}_t$ are respectively denoted as $g_{t,\min}=\min\{|\mathbf{g}_{t,m}|\}$ and $g_{t,\max}=\max\{|\mathbf{g}_{t,m}|\}$. The quantized knob $n_{t,i}$ can be expressed as
\begin{equation}
    n_{t,i}=(g_{t,\min}+u_i\frac{g_{t,\max}-g_{t,\min}}{2^{q}-1}),
\end{equation}
where $u_i\in\{0, 1,...,2^{q}-1\}$.

Thus, the interval $[g_{t,\min}, g_{t,\max}]$ can be divided in to $2^{q}-1$ intervals. For $|g_{t,m}|$ within the interval $N_i=[n_{t, i-1}, n_{t, i})$, it can be quantized as
\begin{equation}\label{B}
    \left.Q(g_{t,m})=\left\{\begin{array}{ll}\operatorname{sign}(g_{t,m})\cdot n_{t,i-1},\!&\text{w.p.}\frac{n_{t,i}-|g_{t,m}|}{n_{t,i}-n_{t,i-1}},\\\operatorname{sign}(g_{t,m})\cdot n_{t,i},&\text{w.p.}\frac{|g_{t,m}|-n_{t,i-1}}{n_{t,i}-n_{t,i-1}},\end{array}\right.\right.
\end{equation}
where w.p. represents for ``with probability". 
Thus, the gradient for the $k$-th client is quantized as $\Tilde{\mathbf{g}}'_{\mathbf{c},k,t}=Q(\Tilde{\mathbf{g}}_{\mathbf{c},k ,t})$, according to \eqref{B}. Then, its model is updated as 
\begin{equation}\label{wc}
    \mathbf{w}_{\mathbf{c},k,t+1}=\Tilde{\mathbf{w}}_{\mathbf{c},k,t}-\eta\Tilde{\mathbf{g}}'_{\mathbf{c},k,t},
\end{equation}
with pruned weight and quantized gradient.

\subsubsection{Client-Side Model Aggregation}
To attain a global model, the client-side models are uploaded to the server for aggregation every $I$ rounds. The aggregated model is then broadcasted to each client as the initial model for the next round. The specific procedures are as follows:
\begin{itemize}
	\item Model Uploading:
	Each client uploads its own client-side model $\mathbf{w}_{\mathbf{c},k,t+1}$ to the server.
	\item Model Aggregation:
	The server aggregates all received models to obtain a global model, denoted as 
	\begin{equation}
		\mathbf{w}_{\mathbf{c},t+1}=\frac1K\sum_{k=1}^K\mathbf{w}_{\mathbf{c},k,t+1}.
	\end{equation}
	\item Model Downloading:
	After completing the client-side model aggregation, the server transmits the aggregated model $\mathbf{w}_{\mathbf{c},t+1}$ to all clients.
\end{itemize}

Due to the varying structures of client-side models after pruning, the sparsity of the aggregated model may not meet the requirement $\rho_t$ for the current round. However, clients do not perform pruning upon receiving the aggregated model immediately. Instead, they start the next round of training and re-evaluate the importance of the aggregated weights.


\section{Convergence Analysis}\label{con_ana}
In this section, we provide the convergence analysis of the proposed lightweight FedSL by characterizing the effect of aggregation frequency, pruning rate, quantization accuracy and split layer selection.
We make several assumptions on the loss functions, model weights and gradients as follows.
\begin{assumption}\label{asp1}
	\textit{The loss function $F(\mathbf{w})$ is differentiable and $\beta$-smooth.}
	\begin{equation}
		\left\|{\nabla F(\mathbf{w}_1)-\nabla F(\mathbf{w}_2)}\right\| \le \beta\left\|{\mathbf{w}_1-\mathbf{w}_2}\right\|,\forall \mathbf{w}_1, \forall \mathbf{w}_2.
	\end{equation}
\end{assumption}

\begin{assumption}\label{asp2}
	\textit{The stochastic gradients are unbiased.}
	\begin{equation}
		\mathbb{E}[\mathbf{g}(\mathbf{w})]=\nabla F(\mathbf{w}),\forall \mathbf{w}.
	\end{equation}
\end{assumption}

\begin{assumption}\label{asp3}
	\textit{(Bounded variance).} \textit{The variance of stochastic gradients has an upper bound.}
	\begin{equation}
		\mathbb{E} \Vert \mathbf{g}(\mathbf{w})-\nabla F(\mathbf{w})\Vert^2 \leq \sum\limits_{l=1}^L \sigma_l^2, \forall \mathbf{w}.
	\end{equation}
\end{assumption}

\begin{assumption}\label{asp4}
	\textit{(Bounded gradient and weight).} \textit{The second moments of the stochastic gradients and weights have upper bound.}
	\begin{equation}
		\mathbb{E} \Vert \mathbf{g}(\mathbf{w}) \Vert^2 \leq \sum\limits_{l=1}^L G_l^2,~
		\mathbb{E} \Vert \mathbf{w} \Vert^2 \leq \sum\limits_{l=1}^L W_l^2,~
		\forall \mathbf{w}.
	\end{equation}
\end{assumption}

\begin{assumption}\label{asp5}
	\textit{According to \cite{stich2018sparsified}, the model error caused by pruning operations is denoted as}
	\begin{equation}
		\mathbb{E} \Vert \mathbf{w}-\Tilde{\mathbf{w}} \Vert^2 \leq \rho \mathbb{E} \Vert \mathbf{w} \Vert^2 \leq \rho \sum\limits_{l=1}^L W_l^2, \forall \mathbf{w},
	\end{equation}
where $\rho$ denotes the pruning rate.
\end{assumption}

\begin{assumption}\label{asp6}
	\textit{Let $\Delta_g = \sqrt{\frac{M}{4}(g_{\max}-g_{\min})^{2}}$, the quantized gradient satisfies \cite{Wang_2022}}
	\begin{align}
		&\mathbb{E}||Q(\nabla F(\mathbf{w}))-\nabla F(\mathbf{w})||^2 \leq \left(\frac{\Delta_g}{2^{q}-1}\right)^2=\sum \limits_{l=1}^L J_{l}^2, \nonumber \\
		&\mathbb{E}[Q(\nabla F(\mathbf{w}))]=\nabla F(\mathbf{w}), \forall \mathbf{w},
	\end{align}
where $q$ denotes the quantized bits.
\end{assumption}

\begin{lemma}\label{rhobound}
	For the target sparsity (pruning rate) for each round $\rho_t=\rho_f + (\frac{t}{T}-1)^3 \rho_f$, we have $\rho_t \leq \rho_f$. Over $T$ rounds, the bound of the sum of pruning rate $\rho_t$ can be represented as 
	$\sum\limits_{t=1}^{T} \rho_t< T \rho_f$.
\end{lemma}

\begin{lemma}\label{lemma}
Under the Assumption 4 and 5, we have
\begin{align*}
	\mathbb{E} \Vert \mathbf{w}_{\mathbf{c},t}\!-\!\Tilde{\mathbf{w}}_{\mathbf{c},k,t} \Vert^2 \!\!\leq\! 8\eta^{2}(I\!+\!1)^{2}\!\sum\limits_{l=1}^{L_c}\!G_l^2
	\!\!+\!4\!\sum\limits_{l=1}^L \!W_l^2 \!\!+\!2\rho_{t}\!\sum\limits_{l=1}^{L_c}\!W_l^2\!.
\end{align*}
\end{lemma}

\begin{proof}
See Appendix A.
\nolinebreak\hspace*{\fill}~\qed\par\unskip\vskip3pt
\end{proof}

\begin{theorem}
Under the  Assumptions 1-6 and Lemma 1-2, if $0 < \eta \leq \frac{1}{2\beta}$, then for all $T \geq 1$, we have
\begin{align}\label{theorem}
    & \frac{1}{T}\sum\limits_{t=1}^{T}\mathbb{E} \Vert \nabla F(\mathbf{w}_t) \Vert^2 \nonumber \\ 
    & < \frac{2 \vartheta}{\eta T} + \sum\limits_{l=1}^{L} (\frac{\beta \eta}{K}\sigma_l^2 + \frac{1}{\eta}G_l^2 + \frac{4(4\beta^{2}+1)}{\eta}W_l^2) \nonumber \\
    &\overbrace{
    +\sum\limits_{l=1}^{L_c}(\frac{\beta \eta}{K}\sigma_l^2 + \underbrace{\frac{(4 \beta^2 + 1)(8\eta^{2} {(I+1)^2} + 1)}{\eta}  G_l^2}_{\rm effect~of ~aggregation ~frequency}}^{\rm effect~of~split~layer~selection} \nonumber \\ 
    &+ \underbrace{\frac{\rho_{f}(4K\beta^2+K+\beta)}{K\eta}W_l^2}_{\rm effect ~of ~pruning ~rate} + \frac{4}{\eta} J_l^2),
\end{align}
where $\vartheta=F(\mathbf{w}_1)-F(\mathbf{w}*)$, $F(\mathbf{w}*)$ represents the minimum value of the loss function.
\end{theorem}

\begin{proof}
	See Appendix B.
	\nolinebreak\hspace*{\fill}~\qed\par\unskip\vskip3pt
\end{proof}

\begin{remark}
    Theorem 1 reveals the impact of the split layer $L_c$, aggregation frequency $I$, preset pruning rate $\rho_f$, and quantization precision $J_l^2$ on the optimum gap after $T$ rounds. Specifically, a shallower split layer $L_c$ reduces the model size for compression and periodic aggregation, leading to faster convergence rate. Additionally, more frequent aggregation (i.e., smaller $I$) is more conducive to the convergence of the global model.
\end{remark}

\section{Simulation Results}
To evaluate the performance of our proposed lightweight FedSL framework, we compare the lightweight FedSL with the traditional FedSL, which does not involve any model compression and performs model aggregation in each round. We consider $K$ clients and one server collaboratively training the VGG-19 model on CIFAR-10 dataset. 
The distance $d$ between clients and the server varied between $100$ to $300$ meters, with transmission power set at $23$ dBm for clients and $37$ dBm for the server, and the subchannel bandwidth is $B=5$ MHz. The power of Gaussian white noise is $-174$ dBm/Hz, while the path loss model is $128.1+37.6\log_{10}(d)$.

Fig.~\ref{fig:PQ} is the ablation experiments for the preset pruning rate $\rho_f$ and quantized bits $q$ with the number of clients and split layer are uniformly set to  $K=5$ and $L_c=8$, respectively. Neither periodic aggregation nor dropout processes are included.
Fig.~\ref{R} illustrates that moderate pruning can lead to faster convergence rate and higher test accuracy. Due to the significant parameter redundancy in the VGG-19 model relative to the CIFAR-10 dataset, pruning less important weights simplifies the network structure, allowing the model to focus on learning critical features. However, when the pruning rate is excessively high (e.g., $\rho_f=0.7$), substantial divergence emerge among client-side models. The removal of crucial features diminishes the consistency of the global model's learning process, resulting in notable fluctuations in the accuracy curve, which indicates performance instability after aggregation. 
Fig.~\ref{Q} illustarates the effect of quantization on model performance. Low-bit quantization (e.g., $4$ bits) results in a clear drop in accuracy, primarily because the substantial loss of gradient information impedes precise parameter updates. In contrast, using an $8$-bit quantization slightly enhances performance by introducing a controlled level of noise, which helps mitigate overfitting as a form of regularization.

Fig.~\ref{fig:ILKP} illustrates the impact of different aggregation frequencies, split layer selection, the number of clients and the dropout rate on model performance under the effects of pruning and quantization. 
Fig.~\ref{I} shows that periodic aggregation comes at the cost of reduced convergence accuracy for the global model since the compression process affects the model structure of each client. When the aggregation interval is $I=1$, the regularization effect of pruning and quantization leads to improved model accuracy. In contrast, when the aggregation intervals $I$ are set to $5$ and $10$, less frequent aggregation results in lower convergence accuracy.
Fig.~\ref{L} illustrates that, under the combined effects of model compression, periodic aggregation and dropout, selecting shallower split layers (e.g., $L_c=4$ or $L_c=8$) can accelerate convergence speed. Additionally, shallower split layers demonstrate stronger robustness against dropout. Although the shallower split layers produce more activation values transmitted through wireless channel, the dropout operation can eliminate redundant information, saving communication resources without significant loss in accuracy.
Moreover, the robustness can also be attributed to the greater capacity of the server-side model to compensate for the missing activations during the forward pass. This compensatory effect helps maintain a stable convergence process despite the reduced data transmission.
For deeper split layers (e.g., $L_c=12$), the dimensions extracted by the convolutional layers are of higher order, and randomly dropping $30\%$ of the activation values introduces greater instability, leading to larger errors on certain batches and a sudden drop in testing accuracy. However, despite the performance fluctuations caused by dropout, the model gradually learns to adapt to this instability over sufficient iterations.
Fig.~\ref{K} illustrates the impact of the number of clients on the global model performance. As the number of clients increases, the convergence of the global model deteriorates. Although lightweighting techniques reduce computational and communication overhead, they also introduce approximation errors. The errors from multiple clients accumulate during aggregation as the number of clients grows, which compromises the convergence of the global model. This result suggests the need for developing lightweighting methods that are independent of the number of clients, enabling scalable deployment of AI models in wireless networks.
Fig.~\ref{P} illustrates the impact of dropout rate on model performance when the split layer is $L_c=8$. The result shows that when the split layer is appropriately chosen, even a relatively high dropout rate (e.g., $p_i=0.7$) only leads to a reduction in convergence rate, without causing a sudden performance drop on specific batches, which is observed in Fig.~\ref{L} when the split layer is $L_c=12$ and the dropout rate is $p_i=0.3$.

\begin{figure}[t]
	\setlength\abovecaptionskip{3pt}
	\centering
	\subfigure[Impact of pruning.]{
		\includegraphics[height=3.7cm,width=4.1cm]{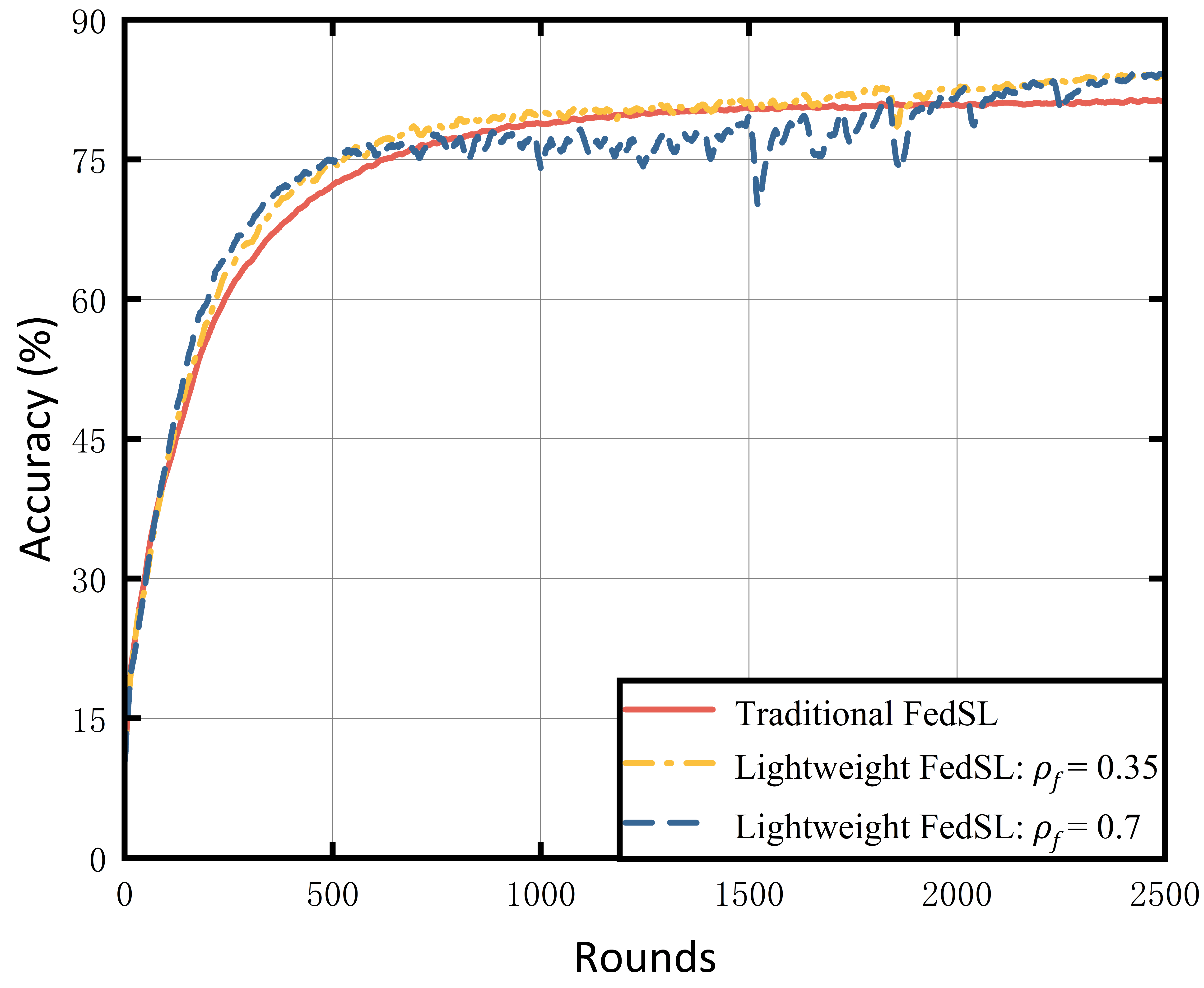}
		\label{R}
	}
	\subfigure[Impact of quantization.]{
		\includegraphics[height=3.7cm,width=4.1cm]{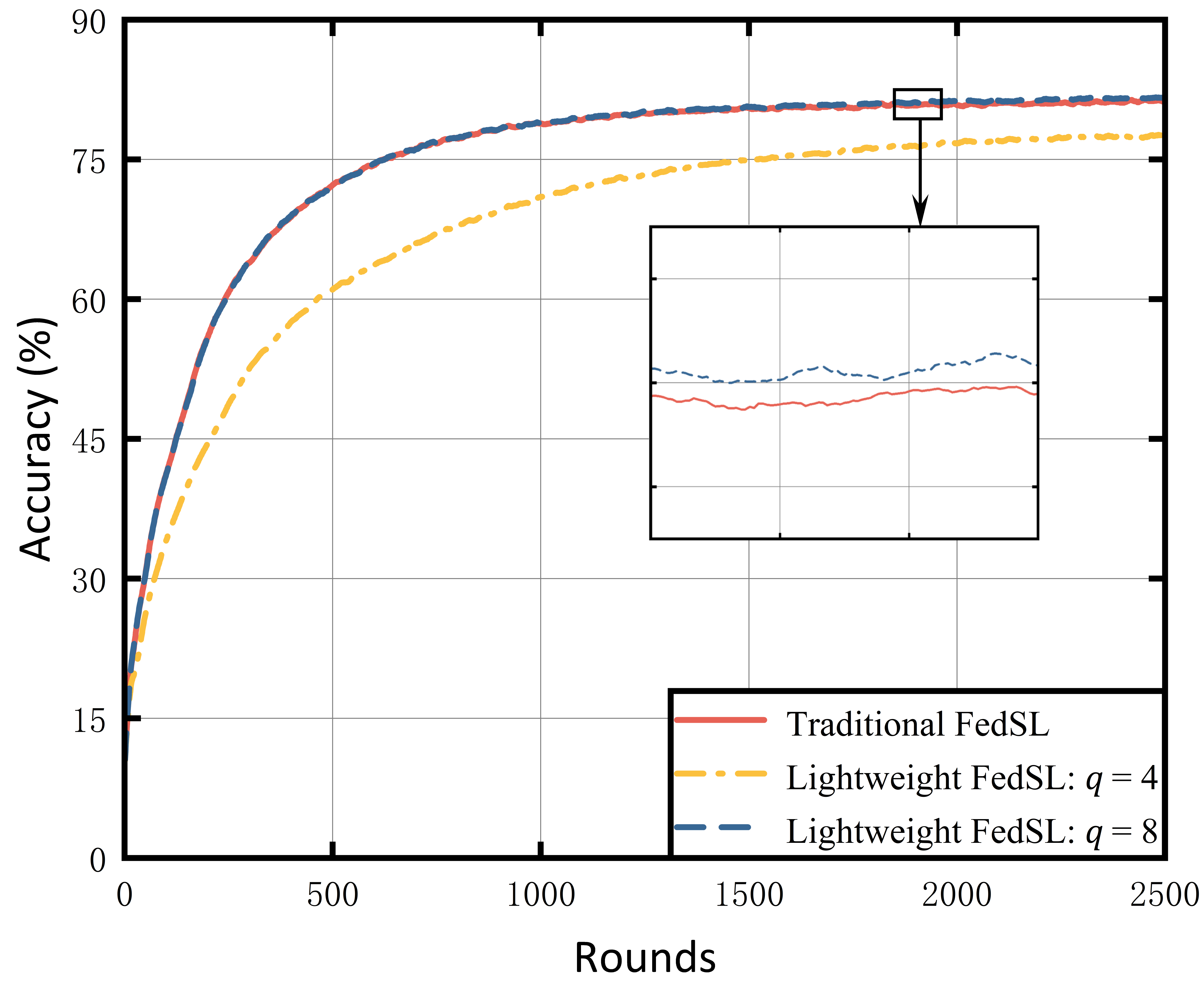}
		\label{Q}
	}
	\caption{Impact of pruning and quantization on performance without periodic aggregation and dropout.}
	\label{fig:PQ}
	\vspace{-2ex}
\end{figure}

Fig.~\ref{T} illustrates the impact of dropout rate on communication latency with varying split layers. Shallower split layers generate a substantial amount of activation values, which can be a major source of communication overhead, especially over bandwidth-constrained wireless channels. Dropout plays a crucial role in significantly reducing the number of activations that need to be transmitted, thereby lowering the communication latency.
Deeper split layers generate fewer activation values, naturally resulting in lower communication costs. However, deeper split layers require careful tuning of the dropout rate, as they are more susceptible to performance degradation due to the higher order features and the reduced resilience to the loss of neural connections.

\begin{figure*}[t]
	\setlength\abovecaptionskip{3pt}
	\centering
	\subfigure[Impact of aggregation frequency with $L_c=8$, $K=5$ and $p_i=0.3$.]{
		\includegraphics[scale=0.18]{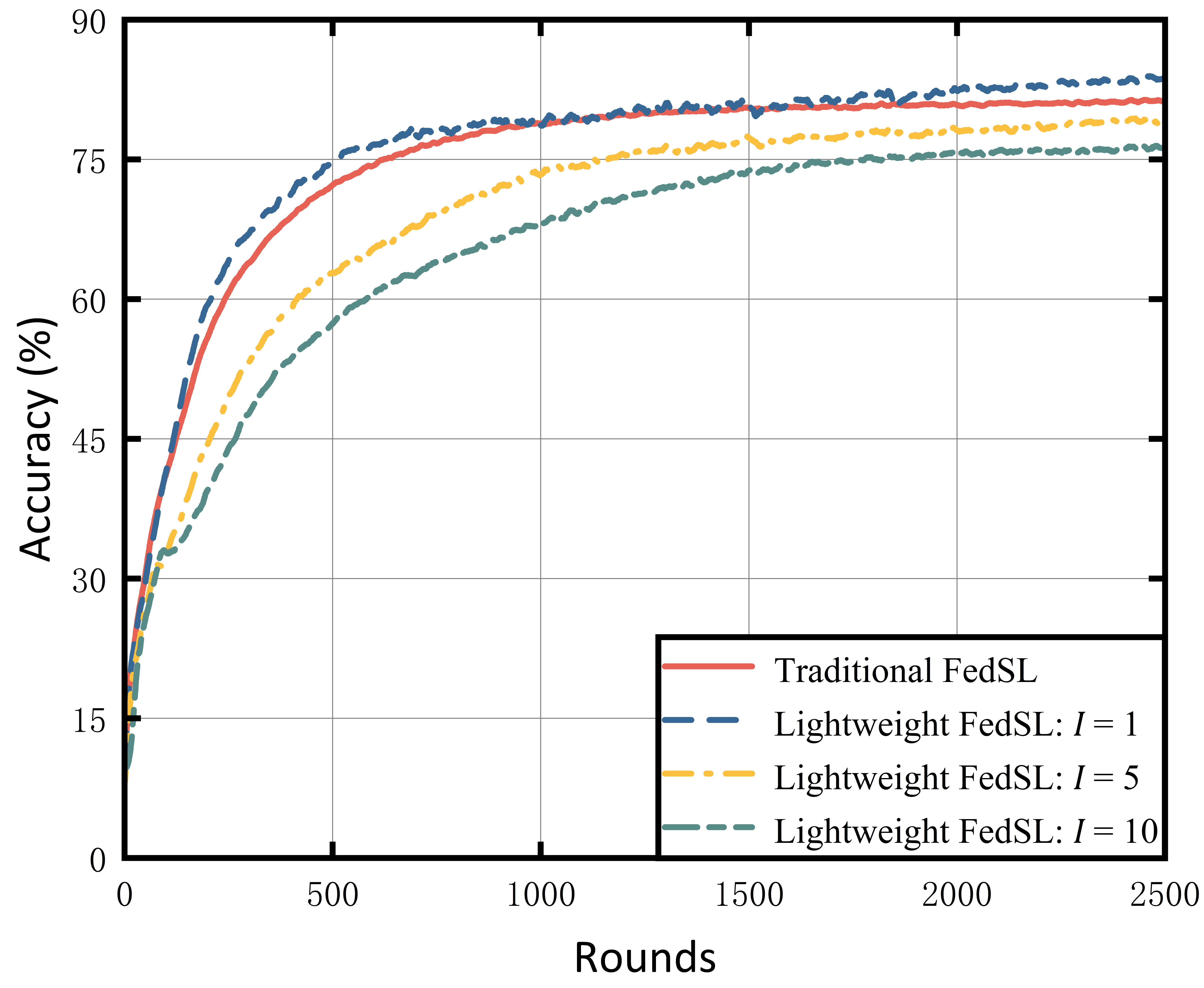}
		\label{I}
	}
	\quad
	\subfigure[Impact of split layer selection with $I=5$, $K=5$ and $p_i=0.3$.]{
		\includegraphics[scale=0.18]{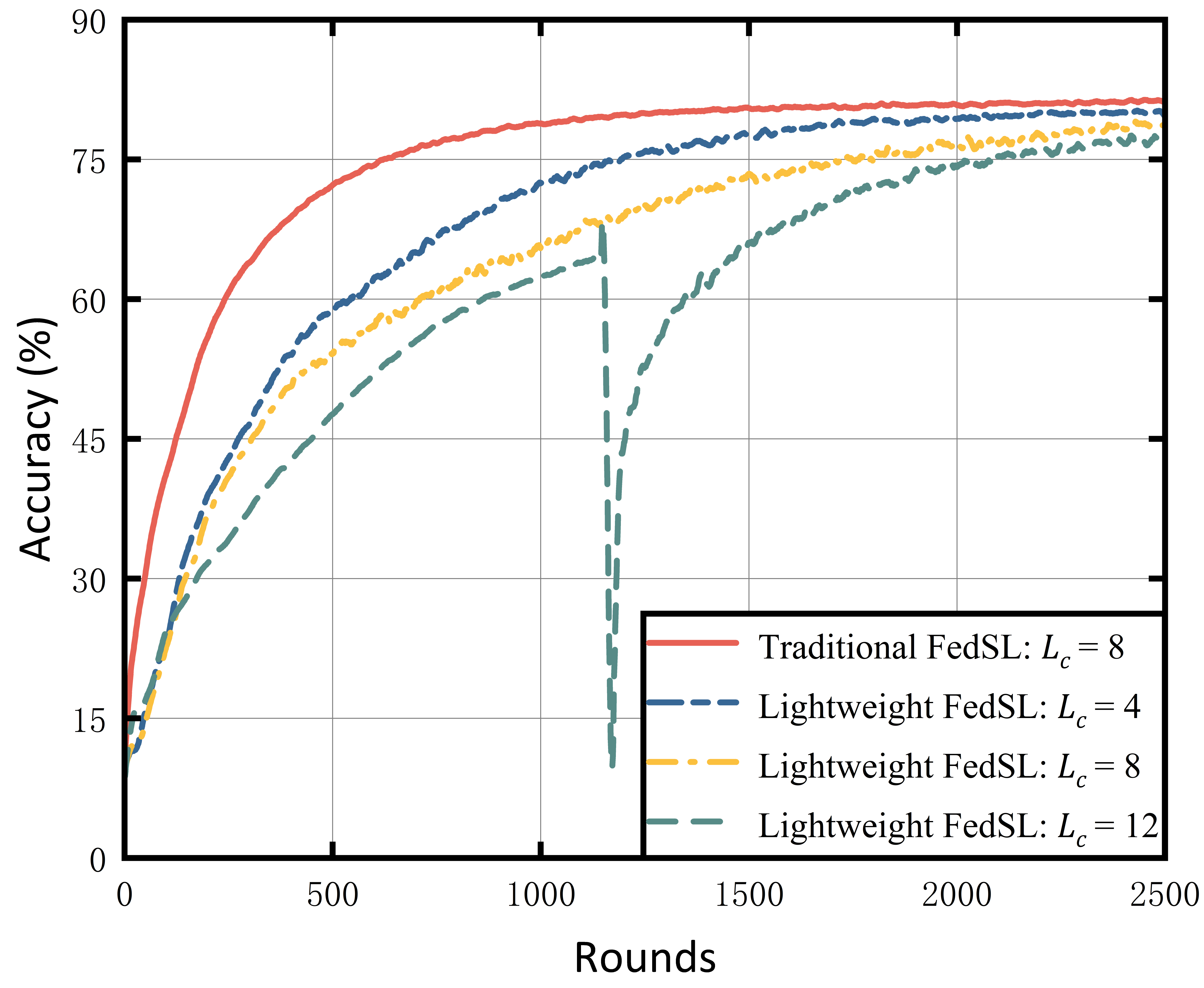}
		\label{L}
	}
	\quad
	\subfigure[Impact of number of clients with $L_c=8$, $I=5$ and $p_i=0.3$.]{
		\includegraphics[scale=0.18]{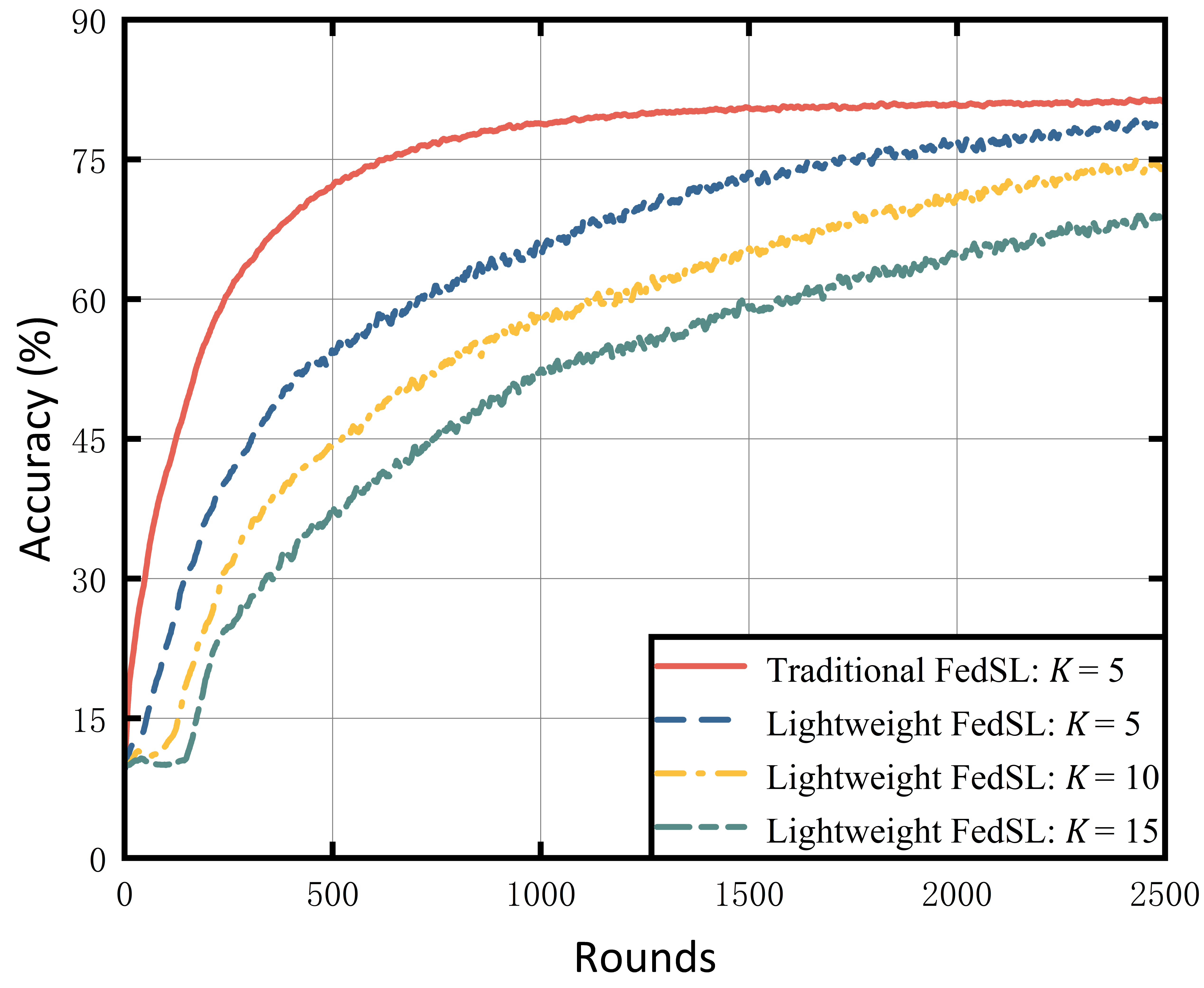}
		\label{K}	
	}
	\quad
	\subfigure[Impact of dropout rate with $L_c=8$, $I=5$ and $K=5$.]{
		\includegraphics[scale=0.18]{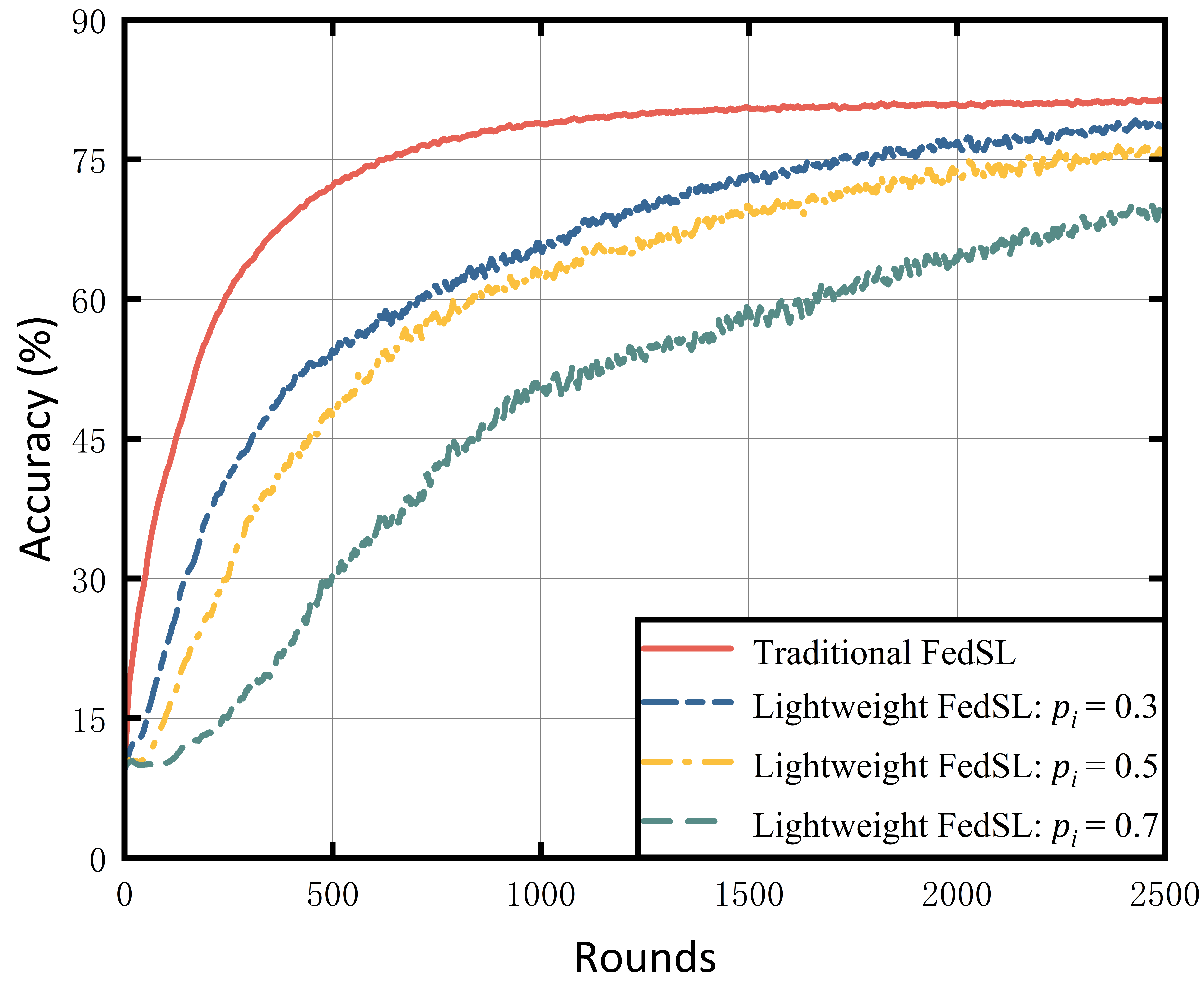}
		\label{P}	
	}
	\caption{Impact of aggregation frequency $I$, split layer selection $L_c$, number of clients $K$ and dropout rate $p_i$ on performance with pruning rate $\rho_f = 0.35$ and quantized bits $q=8$.}
	\label{fig:ILKP}
	\vspace{-2ex}
\end{figure*}

\begin{figure}
	\centering
	\includegraphics[scale=0.2]{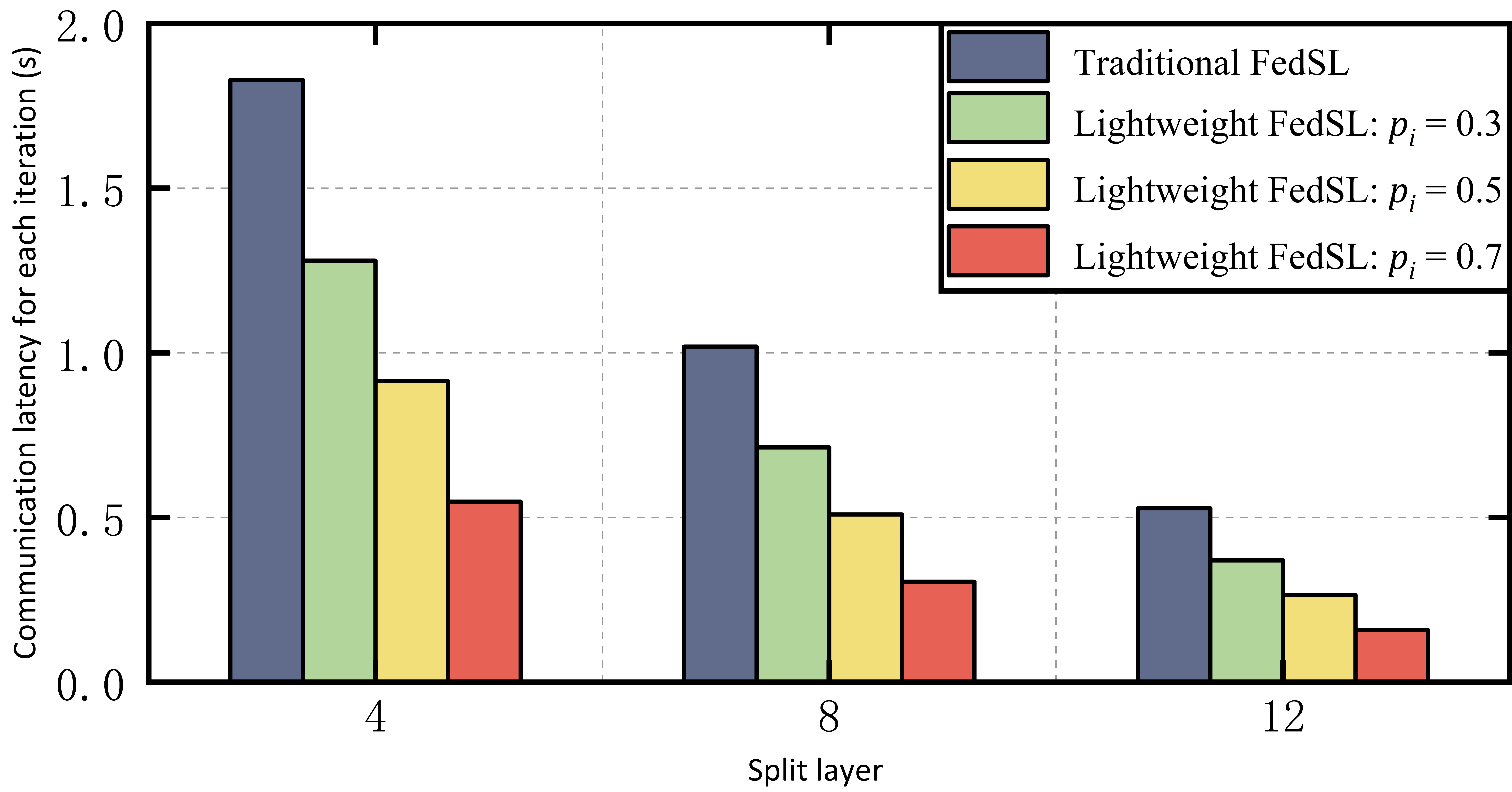}
	\caption{Impact of dropout rate on latency with varying split layer selection under pruning rate $\rho_f=0.35$, quantized bits $q=8$, aggregation frequency $I=5$ and number of clients $K=5$.}
	\label{T}
	\vspace{-2ex}
\end{figure}

\section{Conclusion}
In this letter, we proposed a lightweight FedSL framework based on model pruning and gradient quantization to alleviate the computational and storage burdens on resource-constrained devices. By dropouting the activations transmitted over the channel, the framework conserved communication resources. To evaluate the performance of the proposed framework, we mathematically analyzed the impact of the above factors on convergence performance. Simulation results demonstrated that moderate pruning and quantization not only accelerated training but also enhanced model performance. Shallower split layers exhibited stronger robustness to dropout, which conserved communication resources with only a slight loss in convergence accuracy.

\appendix

\subsection{Proof of Lemma 2}
Fix training round $t\geq 1$. Considering the largest $t_{0} \leq  t$ that satisfies $t_0 \bmod I = 0$ (Note that such $t_{0}$ must exist and $t - t_{0} \leq I$.) Recalling $\mathbf{w}_{\mathbf{c},k,t+1}=\Tilde{\mathbf{w}}_{\mathbf{c},k,t}-\eta\Tilde{\mathbf{g}}'_{\mathbf{c},k,t}$ and $\mathbf{w}_{\mathbf{c},t+1}=\frac1K\sum_{k=1}^K\mathbf{w}_{\mathbf{c},k,t+1}$ for client-side model updating and aggregation, using $\mathbf{m}_t$ to represent the binary matrix obtained by aggregating $\mathbf{m}_{k,t}$ and performing element-wise normalization, we have
\begin{equation}\label{wckt}
	\mathbf{w}_{\mathbf{c},k,t} = \mathbf{m}_{k,t} \odot (\mathbf{w}_{\mathbf{c},t_0}-\eta\sum\limits_{\tau=t_0}^{t} \mathbf{g}'_{\mathbf{c},k,\tau})
\end{equation}
and
\begin{equation}\label{wct}
	\mathbf{w}_{\mathbf{c},t}=\mathbf{m}_t \odot \mathbf{w}_{\mathbf{c},t_0}-\eta\sum\limits_{\tau=t_{0}}^{t}\frac{1}{K}\sum\limits_{k=1}^{K}\mathbf{m}_{k,t} \odot \mathbf{g}'_{\mathbf{c},k,\tau}.
\end{equation}
where Eqn.~\eqref{wckt} follows from
\begin{align*}
	\mathbf{w}_{\mathbf{c},k,t} =& (((\mathbf{w}_{\mathbf{c},k,t_0}-\eta\mathbf{g}'_{\mathbf{c},k,t_0}) \odot \mathbf{m}_{k,t_0}-\eta\mathbf{g}'_{\mathbf{c},k,t_1}) \\
	& \odot \mathbf{m}_{k,t_1}-\eta\mathbf{g}'_{\mathbf{c},k,t_2}) \odot ... \odot \mathbf{m}_{k,t}-\eta\mathbf{g}'_{\mathbf{c},k,t} \odot \mathbf{m}_{k,t} \\
	=& \mathbf{w}_{\mathbf{c},k,t_0} \odot \mathbf{m}_{k,t_0} \odot \mathbf{m}_{k,t_1} \odot ... \odot \mathbf{m}_{k,t} \\
	& -\eta \mathbf{g}'_{\mathbf{c},k,t_0} \odot \mathbf{m}_{k,t_0} \odot \mathbf{m}_{k,t_1} \odot ... \odot \mathbf{m}_{k,t} \\
	& -\eta \mathbf{g}'_{\mathbf{c},k,t_1} \odot \mathbf{m}_{k,t_1} \odot \mathbf{m}_{k,t_2} \odot ... \odot \mathbf{m}_{k,t} \\
	& ...\\
	& -\eta \mathbf{g}'_{\mathbf{c},k,t} \odot \mathbf{m}_{k,t} \\
	\overset{(a)}{=}& \mathbf{m}_{k,t} \odot (\mathbf{w}_{\mathbf{c},t_0}-\eta\sum\limits_{\tau=t_0}^{t} \mathbf{g}'_{\mathbf{c},k,\tau}),
\end{align*}
where (a) follows from $\mathbf{m}_{k,t_0} \odot \mathbf{m}_{k,t_1} \odot ... \odot \mathbf{m}_{k,t}=\mathbf{m}_{k,t_1} \odot \mathbf{m}_{k,t_2} \odot ... \odot \mathbf{m}_{k,t}=...=\mathbf{m}_{k,t}$.

Thus, we have
\begin{align*}
	&\mathbb{E} \Vert \mathbf{w}_{\mathbf{c},t}-\Tilde{\mathbf{w}}_{\mathbf{c},k,t} \Vert^2 \\
	=&\mathbb{E} \Vert \mathbf{w}_{\mathbf{c},t}-\mathbf{w}_{\mathbf{c},k,t}+\mathbf{w}_{\mathbf{c},k,t}-\Tilde{\mathbf{w}}_{\mathbf{c},k,t} \Vert^2 \\
	\overset{(a)}{\leq}& 2\mathbb{E} \Vert \mathbf{w}_{\mathbf{c},t}-\mathbf{w}_{\mathbf{c},k,t} \Vert^2 + 2\mathbb{E} \Vert \mathbf{w}_{\mathbf{c},k,t}-\Tilde{\mathbf{w}}_{\mathbf{c},k,t} \Vert^2 \\
	\overset{(b)}{\leq}& 2\mathbb{E} \Vert (\mathbf{m}_{t}-\mathbf{m}_{k,t}) \odot \mathbf{w}_{\mathbf{c},t}-\eta (\sum\limits_{\tau=t_{0}}^{t}\frac{1}{K}\sum\limits_{k=1}^{K} \mathbf{m}_{k,t} \odot \mathbf{g}'_{\mathbf{c},k,\tau}\\
	& -\sum\limits_{\tau=t_{0}}^{t} \mathbf{m}_{k,t} \odot \mathbf{g}'_{\mathbf{c},k,\tau}) \Vert^2 + 2\rho_t\sum\limits_{l=1}^{L_c}W_l^2 \\
	\overset{(c)}{\leq} &4\eta^2\mathbb{E} \Vert \sum\limits_{\tau=t_{0}}^{t} (\frac{1}{K}\sum\limits_{k=1}^{K} \hat{\mathbf{g}}'_{\mathbf{c},k,\tau})-\sum\limits_{\tau=t_{0}}^{t} \hat{\mathbf{g}}'_{\mathbf{c},k,\tau} \Vert^2 \\
	& +4\mathbb{E}\Vert (\mathbf{m}_{t}-\mathbf{m}_{k,t}) \odot \mathbf{w}_{\mathbf{c},t} \Vert^2+2\rho_{t}\sum\limits_{l=1}^{L_c}W_l^2 \\
	\overset{(d)}{\leq} & 4\eta^2\mathbb{E} \Vert \sum\limits_{\tau=t_{0}}^{t} (\frac{1}{K}\sum\limits_{k=1}^{K} \hat{\mathbf{g}}'_{\mathbf{c},k,\tau})-\sum\limits_{\tau=t_{0}}^{t} \hat{\mathbf{g}}'_{\mathbf{c},k,\tau} \Vert^2 \\
	& +4\sum\limits_{l=1}^{L}W_l^2 + 2\rho_t\sum\limits_{l=1}^{L_c}W_l^2,
\end{align*}
where $\hat{\mathbf{g}}'_{\mathbf{c},k,\tau}=\mathbf{m}_{k,t} \odot \mathbf{g}'_{\mathbf{c},k,\tau}$ and (a) (c) (d) follows by using the inequality $\Vert \sum\limits_{i=1}^{n} \mathbf{z}_i \Vert^2 \leq n \sum\limits_{i=1}^{n}\Vert \mathbf{z}_i \Vert^2 $ for any vectors $\mathbf{z}_i$ and any positive integer n (using $n=2$ in (a) and (c), $n=K$ in (d)). (b) follows from Eqn.~\eqref{wckt}, Eqn.~\eqref{wct} and {{Assumption} 4}.

Note that
\begin{align*}
	& \mathbb{E} \Vert \sum\limits_{\tau=t_{0}}^{t} (\frac{1}{K}\sum\limits_{k=1}^{K} \hat{\mathbf{g}}'_{\mathbf{c},k,\tau})-\sum\limits_{\tau=t_{0}}^{t} \hat{\mathbf{g}}'_{\mathbf{c},k,\tau} \Vert^2 \\
	\overset{(a)}{\leq} & 2\mathbb{E}\{ \Vert \sum\limits_{\tau=t_{0}}^{t} \frac{1}{K}\sum\limits_{k=1}^{K} \hat{\mathbf{g}}'_{\mathbf{c},k,\tau} \Vert^2 + \Vert \sum\limits_{\tau=t_{0}}^{t} \hat{\mathbf{g}}'_{\mathbf{c},k,\tau} \Vert^2 \} \\
	\overset{(b)}{\leq} & 2(t-t_0+1) \mathbb{E} \{ \sum\limits_{\tau=t_{0}}^{t} \Vert \frac{1}{K}\sum\limits_{k=1}^{K} \hat{\mathbf{g}}'_{\mathbf{c},k,\tau} \Vert^2 + \sum\limits_{\tau=t_{0}}^{t} \Vert \hat{\mathbf{g}}'_{\mathbf{c},k,\tau} \Vert^2 \} \\
	\overset{(c)}{\leq} & 2(t-t_0+1) \mathbb{E} \{ \sum\limits_{\tau=t_{0}}^{t} \frac{1}{K}\sum\limits_{k=1}^{K} \Vert \hat{\mathbf{g}}'_{\mathbf{c},k,\tau} \Vert^2 + \sum\limits_{\tau=t_{0}}^{t} \Vert \hat{\mathbf{g}}'_{\mathbf{c},k,\tau} \Vert^2 \} \\
	\overset{(d)}{\leq} & 2(I+1)^{2} \sum\limits_{l=1}^{L_c}G_l^2,
\end{align*}
where, (a)-(c) follows by using the inequality $\Vert \sum\limits_{i=1}^{n} \mathbf{z}_i \Vert^2 \leq n \sum\limits_{i=1}^{n}\Vert \mathbf{z}_i \Vert^2 $ and $n=2$ for (a), $n=(t-t_0+1)$ for (b), and $n=K$ for (c); (d) follows from the {{Assumption} 4}.

Thus, we have
\begin{align*}
	\mathbb{E} \Vert \mathbf{w}_{\mathbf{c},t}\!-\!\Tilde{\mathbf{w}}_{\mathbf{c},k,t} \Vert^2 \!\!\leq\! 8\eta^{2}(I\!+\!1)^{2}\sum\limits_{l=1}^{L_c}\!G_l^2
	\!+\!4\sum\limits_{l=1}^L \!W_l^2 \!+\!2\rho_{t}\sum\limits_{l=1}^{L_c}\!W_l^2.
\end{align*}

\subsection{Proof of the Theorem 1}
For training round $t \leq 1$. By the smoothness of loss function $F$, we have
\begin{align}\label{final}
	\mathbb{E}[F(\mathbf{w}_{t+1})] & \leq \mathbb{E}[F(\mathbf{w}_t)]+\mathbb{E}[\langle \nabla F(\mathbf{w}_t),\mathbf{w}_{t+1}-\mathbf{w}_t \rangle] \nonumber \\ 
	& + \frac{\beta}{2} \mathbb{E}[\Vert \mathbf{w}_{t+1}-\mathbf{w}_t \Vert^2]
\end{align}

Note that
\begin{align}\label{cc}
	&\mathbb{E} \Vert \mathbf{w}_{t+1}-\mathbf{w}_{t} \Vert^2 \nonumber \\
	=& \mathbb{E} \Vert [\mathbf{w}_{\mathbf{c},{t+1}};\mathbf{w}_{\mathbf{s},{t+1}}] - [\mathbf{w}_{\mathbf{c},{t}};\mathbf{w}_{\mathbf{s},{t}}] \Vert^2 \nonumber \\
	=& \mathbb{E} \Vert \mathbf{w}_{\mathbf{c},{t+1}}-\mathbf{w}_{\mathbf{c},{t}};\mathbf{w}_{\mathbf{s},{t+1}}-\mathbf{w}_{\mathbf{s},{t}} \Vert^2 \nonumber \\
	=& \mathbb{E} \Vert \mathbf{w}_{\mathbf{c},{t+1}}-\mathbf{w}_{\mathbf{c},{t}} \Vert^2 
	+ \mathbb{E} \Vert \mathbf{w}_{\mathbf{s},{t+1}}-\mathbf{w}_{\mathbf{s},{t}} \Vert^2.
\end{align}
where $\mathbb{E} \Vert \mathbf{w}_{\mathbf{c},{t+1}}-\mathbf{w}_{\mathbf{c},{t}} \Vert^2$ can be bounded as
\begin{align}\label{bb}
	& \mathbb{E} \Vert \mathbf{w}_{\mathbf{c},{t+1}}-\mathbf{w}_{\mathbf{c},{t}} \Vert^2 = \mathbb{E} \Vert \frac{1}{K}\sum\limits_{k=1}^{K}(\mathbf{w}_{\mathbf{c},k,t+1}-\mathbf{w}_{\mathbf{c},k,t}) \Vert^2 \nonumber \\
	& =\frac{1}{K^2} \sum\limits_{k=1}^{K} \mathbb{E} \Vert (\mathbf{w}_{\mathbf{c},k,t+1}-\Tilde{\mathbf{w}}_{\mathbf{c},k,t})+(\Tilde{\mathbf{w}}_{\mathbf{c},k,t}-\mathbf{w}_{\mathbf{c},k,t}) \Vert^2 \nonumber \\
	& \leq \frac{2}{K^2} \sum\limits_{k=1}^{K} (\mathbb{E} \Vert \mathbf{w}_{\mathbf{c},k,t+1}-\Tilde{\mathbf{w}}_{\mathbf{c},k,t} \Vert^2 +\mathbb{E} \Vert \Tilde{\mathbf{w}}_{\mathbf{c},k,t}-\mathbf{w}_{\mathbf{c},k,t} \Vert^2) \nonumber \\
	& \overset{(a)}{\leq} \frac{2}{K^2} \sum\limits_{k=1}^{K}(\eta^2 \mathbb{E}\Vert\Tilde{\mathbf{g}}'_{\mathbf{c},k,t} \Vert^2 + \rho_t \sum\limits_{l=1}^{L_c}W_l^2) \nonumber \\
	& \overset{(b)}{\leq} \frac{2}{K} \sum\limits_{l=1}^{L_c}(\eta^{2}\sigma_l^2+\rho_t W_l^2)+2\eta^2\mathbb{E} \Vert \frac{1}{K} \sum\limits_{k=1}^{K}\nabla'F(\Tilde{\mathbf{w}}_{\mathbf{c},k,t}) \Vert^2,
\end{align}
where (a) follows from $\mathbf{w}_{\mathbf{c},k,t+1}=\Tilde{\mathbf{w}}_{\mathbf{c},k,t}-\eta\Tilde{\mathbf{g}}'_{\mathbf{c},k,t}$, (b) follows from
\begin{align}
	& \mathbb{E} \Vert \Tilde{\mathbf{g}}'_{\mathbf{c},k,t} \Vert^2 \nonumber \\
	\overset{(a)}{=}& \mathbb{E} \Vert  \Tilde{\mathbf{g}}'_{\mathbf{c},k,t} - \nabla'F(\Tilde{\mathbf{w}}_{\mathbf{c},k,t}) \Vert^2 + \mathbb{E} \Vert \nabla'F(\Tilde{\mathbf{w}}_{\mathbf{c},k,t}) \Vert^2 \nonumber \\
	=& \mathbb{E} \Vert Q(\Tilde{\mathbf{g}}_{\mathbf{c},k,t} - \nabla F(\Tilde{\mathbf{w}}_{\mathbf{c},k,t})) \Vert^2 + \mathbb{E} \Vert \nabla'F(\Tilde{\mathbf{w}}_{\mathbf{c},k,t}) \Vert^2 \nonumber \\
	\overset{(b)}{=}&  \mathbb{E} \Vert  \Tilde{\mathbf{g}}_{\mathbf{c},k,t} - \nabla F(\Tilde{\mathbf{w}}_{\mathbf{c},k,t}) \Vert^2 + \mathbb{E} \Vert \nabla F(\Tilde{\mathbf{w}}_{\mathbf{c},k,t}) \Vert^2 \nonumber \\
	=& \mathbb{E} \Vert \mathbf{m}_{k,t} \odot (\mathbf{g}_{\mathbf{c},k,t} - \nabla F(\mathbf{w}_{\mathbf{c},k,t})) \Vert^2 + \mathbb{E} \Vert \nabla F(\Tilde{\mathbf{w}}_{\mathbf{c},k,t}) \Vert^2 \nonumber \\
	\overset{(c)}{\leq}& \sum\limits_{l=1}^{L_c}\sigma_l^2+\mathbb{E} \Vert \nabla'F(\Tilde{\mathbf{w}}_{\mathbf{c},k,t}) \Vert^2,
\end{align}
where (a) follows by the unbiased stochastic gradient {{Assumption} 2} and the definition of variance, i.e., $\mathbb{E}[\Vert \mathbf{x} \Vert^2 ] = \mathbb{E} [\Vert \mathbf{x} - \mathbb{E} [\mathbf{x}] \Vert^2 ] + [\mathbb{E}\Vert\mathbf{x}\Vert]^2$; (b) and (c) follow from {{Assumption} 6} and {{Assumption} 3}, respectively.

Similarly, $\mathbb{E} \Vert \mathbf{w}_{\mathbf{s},{t+1}}-\mathbf{w}_{\mathbf{s},{t}} \Vert^2$ can be bounded as
\begin{flalign}\label{dd}
	&\mathbb{E} \Vert \mathbf{w}_{\mathbf{s},{t+1}}-\mathbf{w}_{\mathbf{s},{t}} \Vert^2 = \eta^2 \mathbb{E} \Vert \frac{1}{K} \sum\limits_{k=1}^{K} \mathbf{g}_{\mathbf{s},k,t} \Vert^2 \nonumber \\
	=&\eta^2 \mathbb{E} \Vert \frac{1}{K} \sum\limits_{k=1}^{K} (\mathbf{g}_{\mathbf{s},k,t}-\nabla F(\mathbf{w}_{\mathbf{s},k,t})) \Vert^2 \nonumber \\
	+&\eta^2 \mathbb{E} \Vert \frac{1}{K} \sum\limits_{k=1}^{K} \nabla F(\mathbf{w}_{\mathbf{s},k,t}) \Vert^2 \nonumber \\
	\overset{(a)}{\leq}& \frac{\eta^2}{K} \sum\limits_{l=L_c+1}^{L}\sigma_l^2 + \eta^2 \mathbb{E} \Vert \frac{1}{K} \sum\limits_{k=1}^{K} \nabla F(\mathbf{w}_{\mathbf{s},k,t}) \Vert^2,
\end{flalign}
where (a) follows from {{Assumption} 3}.

Thus, substituting Eqn.~\eqref{bb} and Eqn.~\eqref{dd} into Eqn.~\eqref{cc}, $\mathbb{E} \Vert \mathbf{w}_{t+1}-\mathbf{w}_{t} \Vert^2$ can be bounded as
\begin{align}\label{term3}
	&\mathbb{E} \Vert \mathbf{w}_{t+1}-\mathbf{w}_{t} \Vert^2 \leq \frac{\eta^2}{K} \sum\limits_{l=1}^{L} \sigma_l^2 + \frac{\eta^2}{K} \sum\limits_{l=1}^{L_c} \sigma_l^2 + \frac{2\rho_t}{K} \sum\limits_{l=1}^{L_c} W_l^2 \nonumber \\
	&+2\eta^2 \mathbb{E} \Vert \frac{1}{K} \sum\limits_{k=1}^{K}\nabla'F(\Tilde{\mathbf{w}}_{\mathbf{c},k,t}) \Vert^2 +\eta^2 \mathbb{E} \Vert \frac{1}{K} \sum\limits_{k=1}^{K} \nabla F(\mathbf{w}_{\mathbf{s},k,t}) \Vert^2
\end{align}
We further note that
\begin{align}\label{ee}
	& \mathbb{E} \langle \nabla F(\mathbf{w}_t), \mathbf{w}_{t+1}-\mathbf{w}_t \rangle \nonumber \\
	=& \mathbb{E} \langle \nabla F(\mathbf{w}_{\mathbf{s},t}), \mathbf{w}_{\mathbf{s},t+1}-\mathbf{w}_{\mathbf{s},t} \rangle
	+\mathbb{E} \langle \nabla F(\mathbf{w}_{\mathbf{c},t}), \mathbf{w}_{\mathbf{c},t+1}-\mathbf{w}_{\mathbf{c},t} \rangle.
\end{align}
The first term can be written as
\begin{align}\label{ff}
	&\mathbb{E} \langle \nabla F(\mathbf{w}_{\mathbf{s},t}), \mathbf{w}_{\mathbf{s},t+1}-\mathbf{w}_{\mathbf{s},t} \rangle \nonumber \\
	=&-\eta \mathbb{E} \langle \nabla F(\mathbf{w}_{\mathbf{s},t}),\frac{1}{K}\sum\limits_{k=1}^{K}\mathbf{g}_{\mathbf{s},k,t} \rangle \nonumber \\
	\overset{(a)}{=}& -\eta \mathbb{E} \langle \nabla F (\mathbf{w}_{\mathbf{c},t}),\frac{1}{K}\sum\limits_{k=1}^{K} \nabla F (\mathbf{w}_{\mathbf{c},k,t}) \rangle \nonumber \\
	\overset{(b)}{=}& -\frac{\eta}{2} \mathbb{E} \Vert \nabla F(\mathbf{w}_{\mathbf{s},t}) \Vert^2 - \frac{\eta}{2} \mathbb{E} \Vert \frac{1}{K}\sum\limits_{k=1}^{K} \nabla F (\mathbf{w}_{\mathbf{s},k,t}) \Vert^2 \nonumber \\
	&+\frac{\eta}{2} \mathbb{E} \Vert \nabla F(\mathbf{w}_{\mathbf{s},t}) - \frac{1}{K}\sum\limits_{k=1}^{K} \nabla F (\mathbf{w}_{\mathbf{s},k,t})\Vert^2,
\end{align}
where (a) follows from {{Assumption} 2}; (b) follows from the identity $\langle \mathbf{a},\mathbf{b} \rangle = \frac{1}{2}(\Vert \mathbf{a} \Vert^2 + \Vert \mathbf{b} \Vert^2 - \Vert \mathbf{a}-\mathbf{b} \Vert^2)$ . For the second term in Eqn.~\eqref{ee}, we have
\begin{align}\label{gg}
	&\mathbb{E} \langle \nabla F(\mathbf{w}_{\mathbf{c},t}), \mathbf{w}_{\mathbf{c},t+1}-\mathbf{w}_{\mathbf{c},t} \rangle \nonumber \\
	\overset{(a)}{=}& \mathbb{E} \langle \nabla F(\mathbf{w}_{\mathbf{c},t}),\frac{1}{K}\sum\limits_{k=1}^{K} (\Tilde{\mathbf{w}}_{\mathbf{c},k,t}-\eta\Tilde{\mathbf{g}}'_{\mathbf{c},k,t})-\mathbf{w}_{\mathbf{c},t} \rangle \nonumber \\
	=& \mathbb{E} \langle \nabla F(\mathbf{w}_{\mathbf{c},t}),-\eta\frac{1}{K}\sum\limits_{k=1}^{K}\Tilde{\mathbf{g}}'_{\mathbf{c},k,t} \rangle \nonumber \\
	&+\mathbb{E} \langle \nabla F(\mathbf{w}_{\mathbf{c},t}),\frac{1}{K}\sum\limits_{k=1}^{K}\Tilde{\mathbf{w}}_{\mathbf{c},k,t}-\mathbf{w}_{\mathbf{c},t} \rangle
\end{align}
where term $\mathbb{E} \langle \nabla F(\mathbf{w}_{\mathbf{c},t}),-\eta\frac{1}{K}\sum\limits_{k=1}^{K}\Tilde{\mathbf{g}}'_{\mathbf{c},k,t} \rangle$ in Eqn.~\eqref{gg} can be written as
\begin{align}\label{hh}
	&\mathbb{E} \langle \nabla F(\mathbf{w}_{\mathbf{c},t}),-\eta\frac{1}{K}\sum\limits_{k=1}^{K}\Tilde{\mathbf{g}}'_{\mathbf{c},k,t} \rangle \nonumber \\
	=& -\eta \mathbb{E} \langle \nabla F(\mathbf{w}_{\mathbf{c},t}),\frac{1}{K}\sum\limits_{k=1}^{K}\Tilde{\mathbf{g}}'_{\mathbf{c},k,t} \rangle \nonumber \\
	=& -\frac{\eta}{2} \mathbb{E} \Vert \nabla' F(\mathbf{w}_{\mathbf{c},t}) \Vert^2 - \frac{\eta}{2} \mathbb{E} \Vert \frac{1}{K}\sum\limits_{k=1}^{K} \nabla' F (\Tilde{\mathbf{w}}_{\mathbf{c},k,t}) \Vert^2 \nonumber \\
	&+\frac{\eta}{2} \mathbb{E} \Vert \nabla F(\mathbf{w}_{\mathbf{c},t}) - \frac{1}{K}\sum\limits_{k=1}^{K} \nabla' F (\Tilde{\mathbf{w}}_{\mathbf{c},k,t})\Vert^2
\end{align}
and term $\mathbb{E} \langle \nabla F(\mathbf{w}_{\mathbf{c},t}),\frac{1}{K}\sum\limits_{k=1}^{K}\Tilde{\mathbf{w}}_{\mathbf{c},k,t}-\mathbf{w}_{\mathbf{c},t} \rangle$ in Eqn.~\eqref{gg} can be bounded as
\begin{align}\label{ii}
	&\mathbb{E} \langle \nabla F(\mathbf{w}_{\mathbf{c},t}),\frac{1}{K}\sum\limits_{k=1}^{K}\Tilde{\mathbf{w}}_{\mathbf{c},k,t}-\mathbf{w}_{\mathbf{c},t} \rangle \nonumber \\
	=& \frac{1}{K} \sum\limits_{k=1}^{K} \mathbb{E} \langle \nabla F(\mathbf{w}_{\mathbf{c},t}),\Tilde{\mathbf{w}}_{\mathbf{c},k,t}-\mathbf{w}_{\mathbf{c},t} \rangle \nonumber \\
	\overset{(a)}{\leq} & \frac{1}{2K} \sum\limits_{k=1}^{K} (\mathbb{E} \Vert \nabla F(\mathbf{w}_{\mathbf{c},t}) \Vert^2 + \mathbb{E} \Vert \Tilde{\mathbf{w}}_{\mathbf{c},k,t}-\mathbf{w}_{\mathbf{c},t} \Vert^2 ) \nonumber \\
	\overset{(b)}{\leq}& \frac{1}{2} ((8\eta^{2}(I+1)^{2}+1)\sum\limits_{l=1}^{L_c}G_l^2+4\sum\limits_{l=1}^L W_l^2 +2\rho_t\sum\limits_{l=1}^{L_c}W_l^2),
\end{align}
where (a) follows by the inequality $\langle \mathbf{a},\mathbf{b} \rangle \leq \frac{1}{2} (\Vert \mathbf{a} \Vert^2 + \Vert \mathbf{b} \Vert^2)$; (b) follows from {{Assumption} 4} and {{Lemma} 2}.

Substituting Eqn.~\eqref{hh} and Eqn.~\eqref{ii} into Eqn.~\eqref{gg}, $\mathbb{E} \langle \nabla F(\mathbf{w}_{\mathbf{c},t}), \mathbf{w}_{\mathbf{c},t+1}-\mathbf{w}_{\mathbf{c},t} \rangle$ can be bounded as
\begin{align}\label{JJ}
	&\mathbb{E} \langle \nabla F(\mathbf{w}_{\mathbf{c},t}), \mathbf{w}_{\mathbf{c},t+1}-\mathbf{w}_{\mathbf{c},t} \rangle \nonumber \\
	\leq&  -\frac{\eta}{2} \mathbb{E} \Vert \nabla F(\mathbf{w}_{\mathbf{c},t}) \Vert^2 - \frac{\eta}{2} \mathbb{E} \Vert \frac{1}{K}\sum\limits_{k=1}^{K} \nabla' F (\Tilde{\mathbf{w}}_{\mathbf{c},k,t}) \Vert^2 \nonumber \\
	&+\frac{\eta}{2} \mathbb{E} \Vert \nabla F(\mathbf{w}_{\mathbf{c},t}) - \frac{1}{K}\sum\limits_{k=1}^{K} \nabla' F (\Tilde{\mathbf{w}}_{\mathbf{c},k,t})\Vert^2 \nonumber \\
	&+ \frac{1}{2} ((8\eta^{2}(I+1)^{2}+1)\sum\limits_{l=1}^{L_c}G_l^2+4\sum\limits_{l=1}^L W_l^2 +2\rho_t\sum\limits_{l=1}^{L_c}W_l^2)
\end{align}
Substituting Eqn.~\eqref{ff} and Eqn.~\eqref{JJ} into Eqn.~\eqref{ee}, we have
\begin{align}\label{term2}
	\mathbb{E} \langle \nabla & F(\mathbf{w}_t), \mathbf{w}_{t+1}-\mathbf{w}_t \rangle \nonumber \\
	\leq  -\frac{\eta}{2}& \{ \mathbb{E} \Vert \nabla F(\mathbf{w}_{\mathbf{s},t}) \Vert^2 + \mathbb{E} \Vert \nabla F(\mathbf{w}_{\mathbf{c},t}) \Vert^2 \} \nonumber \\
	-\frac{\eta}{2}& \{ \mathbb{E} \Vert \frac{1}{K}\sum\limits_{k=1}^{K} \nabla F (\mathbf{w}_{\mathbf{s},k,t}) \Vert^2 + \mathbb{E} \Vert \frac{1}{K}\sum\limits_{k=1}^{K} \nabla' F (\Tilde{\mathbf{w}}_{\mathbf{c},k,t}) \Vert^2 \nonumber \} \nonumber \\
	+\frac{\eta}{2}& \{ \mathbb{E} \Vert \nabla F(\mathbf{w}_{\mathbf{s},t}) - \frac{1}{K}\sum\limits_{k=1}^{K} \nabla F (\mathbf{w}_{\mathbf{s},k,t})\Vert^2 \nonumber \\
	&+ \mathbb{E} \Vert \nabla F(\mathbf{w}_{\mathbf{c},t}) - \frac{1}{K}\sum\limits_{k=1}^{K} \nabla' F (\Tilde{\mathbf{w}}_{\mathbf{c},k,t})\Vert^2  \} \nonumber \\
	+\frac{1}{2}& ((8\eta^{2}(I+1)^{2}+1)\sum\limits_{l=1}^{L_c}G_l^2+4\sum\limits_{l=1}^L W_l^2 +2\rho_t\sum\limits_{l=1}^{L_c}W_l^2)
\end{align}
Then substituting Eqn.~\eqref{term3} and Eqn.~\eqref{term2} into Eqn.~\eqref{final}, we have
\begin{align}\label{final1}
	&\mathbb{E}[F(\mathbf{w}_{t+1})] \nonumber \\
	\leq& \mathbb{E}[F(\mathbf{w}_t)] -\frac{\eta}{2} \mathbb{E} \Vert \nabla F(\mathbf{w}_t) \Vert^2 \nonumber \\
	&- \frac{\eta-\eta^{2}\beta}{2} \mathbb{E} \Vert \frac{1}{K}\sum\limits_{k=1}^{K} \nabla F (\mathbf{w}_{\mathbf{s},k,t}) \Vert^2 \nonumber \\
	&- \frac{\eta-2\eta^{2}\beta}{2} \mathbb{E} \Vert \frac{1}{K}\sum\limits_{k=1}^{K} \nabla' F (\Tilde{\mathbf{w}}_{\mathbf{c},k,t}) \Vert^2 \nonumber \\
	&+ \frac{\beta\eta^2}{2K} \sum\limits_{l=1}^{L} \sigma_l^2 + \frac{\beta\eta^2}{2K} \sum\limits_{l=1}^{L_c} \sigma_l^2 + \frac{\beta\rho_t}{K}\sum\limits_{l=1}^{L_c} W_l^2 
	\nonumber \\
	&+\frac{1}{2} ((8\eta^{2}(I+1)^{2}+1)\sum\limits_{l=1}^{L_c}G_l^2+4\sum\limits_{l=1}^L W_l^2 +2\rho_{t}\sum\limits_{l=1}^{L_c}W_l^2) \nonumber \\
	&+\frac{\eta}{2} \{ \mathbb{E} \Vert \nabla F(\mathbf{w}_{\mathbf{s},t}) - \frac{1}{K}\sum\limits_{k=1}^{K} \nabla F (\mathbf{w}_{\mathbf{s},k,t})\Vert^2 \nonumber \\
	&+ \mathbb{E} \Vert \nabla F(\mathbf{w}_{\mathbf{c},t}) - \frac{1}{K}\sum\limits_{k=1}^{K} \nabla' F (\Tilde{\mathbf{w}}_{\mathbf{c},k,t})\Vert^2 \} \nonumber \\
	\overset{(a)}{\leq}& \mathbb{E}[F(\mathbf{w}_t)] -\frac{\eta}{2} \mathbb{E} \Vert \nabla F(\mathbf{w}_t) \Vert^2 \nonumber \\
	&+ \frac{\beta\eta^2}{2K} \sum\limits_{l=1}^{L} \sigma_l^2 + \frac{\beta\eta^2}{2K} \sum\limits_{l=1}^{L_c} \sigma_l^2 + \frac{\beta\rho_t}{K}\sum\limits_{l=1}^{L_c} W_l^2 
	\nonumber \\
	&+\frac{1}{2} ((8\eta^{2}(I+1)^{2}+1)\sum\limits_{l=1}^{L_c}G_l^2+4\sum\limits_{l=1}^L W_l^2 +2\rho_t\sum\limits_{l=1}^{L_c}W_l^2) \nonumber \\
	&+\frac{\eta}{2} \{ \mathbb{E} \Vert \nabla F(\mathbf{w}_{\mathbf{s},t}) - \frac{1}{K}\sum\limits_{k=1}^{K} \nabla F (\mathbf{w}_{\mathbf{s},k,t})\Vert^2 \nonumber \\
	&+ \mathbb{E} \Vert \nabla F(\mathbf{w}_{\mathbf{c},t}) - \frac{1}{K}\sum\limits_{k=1}^{K} \nabla' F (\Tilde{\mathbf{w}}_{\mathbf{c},k,t})\Vert^2 \} \nonumber \\
	\overset{(b)}{\leq}& \mathbb{E}[F(\mathbf{w}_t)] -\frac{\eta}{2} \mathbb{E} \Vert \nabla F(\mathbf{w}_t) \Vert^2 \nonumber \\
	&+ \frac{\beta\eta^2}{2K} \sum\limits_{l=1}^{L} \sigma_l^2 + \frac{\beta\eta^2}{2K} \sum\limits_{l=1}^{L_c} \sigma_l^2 + \frac{\beta\rho_t}{K}\sum\limits_{l=1}^{L_c} W_l^2 
	+ 2\sum\limits_{l=1}^{L_c} J_l^2  \nonumber \\
	&+(2\beta^{2}+\frac{1}{2}) ((8\eta^{2}(I+1)^{2}+1)\sum\limits_{l=1}^{L_c}G_l^2+4\sum\limits_{l=1}^L W_l^2  \nonumber \\
	& ~~ +2\rho_t\sum\limits_{l=1}^{L_c}W_l^2)
\end{align}
where (a) follows from $0<\eta\leq \frac{1}{2\beta}$ and (b) holds because of the following inequality Eqn.~\eqref{kk} and Eqn.~\eqref{ll}
\begin{align}\label{kk}
	&\mathbb{E} \Vert \nabla F(\mathbf{w}_{\mathbf{s},t}) - \frac{1}{K}\sum\limits_{k=1}^{K} \nabla F (\mathbf{w}_{\mathbf{s},k,t})\Vert^2 \nonumber \\
	=&\mathbb{E} \Vert \frac{1}{K}\sum\limits_{k=1}^{K} \nabla F(\mathbf{w}_{\mathbf{s},t}) - \frac{1}{K}\sum\limits_{k=1}^{K} \nabla F (\mathbf{w}_{\mathbf{s},k,t})\Vert^2 \nonumber \\
	\leq & \frac{1}{K}\sum\limits_{k=1}^{K} \mathbb{E} \Vert \nabla F(\mathbf{w}_{\mathbf{s},t}) -  \nabla F (\mathbf{w}_{\mathbf{s},k,t})\Vert^2 \nonumber \\
	\overset{(a)}{\leq}& \frac{\beta^2}{K} \sum\limits_{k=1}^{K} \mathbb{E} \Vert \mathbf{w}_{\mathbf{s},t}-\mathbf{w}_{\mathbf{s},k,t} \Vert^2 \overset{(b)}{=} 0,
\end{align}
where (a) follows from {{Assumption} 1}; (b) holds because the server-side model of each client is the aggregated version of the whole server-side model. The term $\mathbb{E} \Vert \nabla F(\mathbf{w}_{\mathbf{c},t}) - \frac{1}{K}\sum\limits_{k=1}^{K} \nabla' F (\Tilde{\mathbf{w}}_{\mathbf{c},k,t})\Vert^2$ in Eqn.~\eqref{final1} can be bounded as
\begin{align}\label{ll}
	&\mathbb{E} \Vert \nabla F(\mathbf{w}_{\mathbf{c},t}) - \frac{1}{K}\sum\limits_{k=1}^{K} \nabla' F (\Tilde{\mathbf{w}}_{\mathbf{c},k,t})\Vert^2 \nonumber \\
	\leq & \frac{1}{K}\sum\limits_{k=1}^{K} \mathbb{E} \Vert \nabla F(\mathbf{w}_{\mathbf{c},t}) - \nabla' F (\Tilde{\mathbf{w}}_{\mathbf{c},k,t})\Vert^2 \nonumber \\
	\leq & \frac{1}{K}\sum\limits_{k=1}^{K} \mathbb{E} \Vert \nabla F(\mathbf{w}_{\mathbf{c},t}) - \nabla F (\Tilde{\mathbf{w}}_{\mathbf{c},k,t}) \nonumber \\
	& +\nabla F (\Tilde{\mathbf{w}}_{\mathbf{c},k,t})-\nabla' F (\Tilde{\mathbf{w}}_{\mathbf{c},k,t})\Vert^2 \nonumber \\
	\overset{(a)}{\leq} & \frac{2}{K}\sum\limits_{k=1}^{K} \mathbb{E} \{ \Vert \nabla F(\mathbf{w}_{\mathbf{c},t}) - \nabla F (\Tilde{\mathbf{w}}_{\mathbf{c},k,t}) \Vert^2 + \sum\limits_{l=1}^{L_c}J_l^2 \} \nonumber \\
	\overset{(b)}{\leq} & \frac{2}{K}\sum\limits_{k=1}^{K} \mathbb{E} \{ \beta^2 \Vert \mathbf{w}_{\mathbf{c},t} - \Tilde{\mathbf{w}}_{\mathbf{c},k,t} \Vert^2 + \sum\limits_{l=1}^{L_c}J_l^2 \} \nonumber \\
	\overset{(c)}{\leq} & 2\beta^2((8\eta^{2}(I+1)^{2}+1)\sum\limits_{l=1}^{L_c}G_l^2+4\sum\limits_{l=1}^L W_l^2 \nonumber \\
	&+2\rho_t\sum\limits_{l=1}^{L_c}W_l^2)+2\sum\limits_{l=1}^{L_c}J_l^2
\end{align}
where (a) follows from the inequality $\Vert \sum\limits_{i=1}^{n} \mathbf{z}_i \Vert^2 \leq n \sum\limits_{i=1}^{n}\Vert \mathbf{z}_i \Vert^2 $ and {{Assumption} 6}, (b) follows from {{Assumption} 1} and (c) follows from {{Lemma} 2}.

Rearranging Eqn.~\eqref{final1} and dividing both sides by $\frac{\eta}{2T}$ and summing over $t \in \{1,...,T\}$, the inequality can be written as
\begin{align}
	& \frac{1}{T}\sum\limits_{t=1}^{T}\mathbb{E} \Vert \nabla F(\mathbf{w}_t) \Vert^2 \nonumber \\ 
	\overset{(a)}{<}&  \frac{2 (F(\mathbf{w}_1)-F(\mathbf{w}*))}{\eta T} \nonumber \\
	&+ \sum\limits_{l=1}^{L} (\frac{\beta \eta}{K}\sigma_l^2 + \frac{1}{\eta}G_l^2 + \frac{4(4\beta^{2}+1)}{\eta}W_l^2) \nonumber \\ 
	&+ \sum\limits_{l=1}^{L_c}(\frac{\beta \eta}{K}\sigma_l^2 + \frac{(4 \beta^2 + 1)(8\eta^{2} {(I+1)^2} + 1)}{\eta}  G_l^2 \nonumber \\ 
	&+ \frac{\rho_{f}(4K\beta^2+K+\beta)}{K\eta}W_l^2 + \frac{4}{\eta} J_l^2) \nonumber \\
	\overset{(b)}{\leq}& \frac{2 \vartheta}{\eta T} + \sum\limits_{l=1}^{L} (\frac{\beta \eta}{K}\sigma_l^2 + \frac{1}{\eta}G_l^2 + \frac{4(4\beta^{2}+1)}{\eta}W_l^2) \nonumber \\ 
	&+ \sum\limits_{l=1}^{L_c}(\frac{\beta \eta}{K}\sigma_l^2 + \frac{(4 \beta^2 + 1)(8\eta^{2} {(I+1)^2} + 1)}{\eta}  G_l^2 \nonumber \\ 
	&+ \frac{\rho_{f}(4K\beta^2+K+\beta)}{K\eta}W_l^2 + \frac{4}{\eta} J_l^2),
\end{align}
where (a) follows from {Lemma 1}, (b) follows because $F(\mathbf{w}*)$ is the minimum value of problem.

\section*{Acknowledgement}
The work of Dr. Wanli Ni presented in this paper was conducted during his doctoral studies at Beijing University of Posts and Telecommunications.

\bibliographystyle{IEEEtran}
\bibliography{IEEEabrv,ref}

\end{document}

%% file: header.tex
\newtheorem{theorem}{\bf Theorem}

\newtheorem{lemma}{\bf Lemma}
\newtheorem{corollary}{\bf Corollary}

\newtheorem{assumption}{\bf Assumption}

\newtheorem{remark}{\bf Remark}

\def\qed{$\blacksquare$}

\def\phi{\varphi}

\def\({\left(}
\def\){\right)}

\def\b0{{\mathbf{0}}}

